\documentclass{article}

\usepackage{microtype}
\usepackage{graphicx}
\usepackage{subcaption}
\usepackage{booktabs} %

\usepackage[colorlinks=true, allcolors=blue]{hyperref}

\usepackage[accepted]{icml2026none}

\icmltitlerunning{Anti-causal domain generalization: Leveraging unlabeled data}

\usepackage{amsmath, mathtools}
\usepackage{centernot}
\usepackage{amssymb}
\usepackage{statmath}
\usepackage{graphicx}
\usepackage{amsthm}
\usepackage{bbm}
\usepackage{tabu}
\usepackage{dsfont}
\usepackage{paralist}
\usepackage[inline]{enumitem}
\setlist{topsep=-1pt, itemsep=-1pt}
\usepackage{minitoc}
\usepackage[title,toc,titletoc,page]{appendix}
\usepackage{algorithm,algorithmic}

\renewcommand{\algorithmicrequire}{\textbf{Input:}}
\renewcommand{\algorithmicensure}{\textbf{Output:}}

\usepackage{tikz}
\usepackage{tkz-graph}
\usetikzlibrary{shapes.geometric}
\usetikzlibrary{backgrounds}
\usetikzlibrary{arrows.meta}

\makeatletter

\newcommand{\inlineitem}[1][]{%
\ifnum\enit@type=\tw@
    {\descriptionlabel{#1}}
  \hspace{\labelsep}%
\else
  \ifnum\enit@type=\z@
       \refstepcounter{\@listctr}\fi
    \quad\@itemlabel\hspace{\labelsep}%
\fi}
\usepackage{physics}
\usepackage{caption}
\usepackage{accents}
\usepackage[english]{babel}
\usepackage{aligned-overset}
\usepackage{amsfonts}
\usepackage{booktabs}
\usepackage{tabu}
\usepackage[T1]{fontenc}
\usepackage{mathrsfs}
\usepackage{multirow}
\usepackage{stmaryrd}
\usepackage{tikz}
\usepackage[colorlinks=true, allcolors=blue]{hyperref}

\usepackage{subcaption}

\theoremstyle{plain}
\newtheorem{theorem}{Theorem}[section]
\newtheorem{setting}[theorem]{Setting}
\newtheorem{proposition}[theorem]{Proposition}
\newtheorem{lemma}[theorem]{Lemma}

\theoremstyle{definition}

\theoremstyle{remark}
\newtheorem{remark}[theorem]{Remark}

\theoremstyle{definition}

\newcommand{\tst}{\operatorname{tst}}

\let\argmin\relax

\DeclareMathOperator*{\argmin}{argmin}

\newcommand{\ci}{\mathrel{\perp\mspace{-10mu}\perp}}
\newcommand\numberthis{\addtocounter{equation}{1}\tag{\theequation}}

\DeclareMathOperator{\EX}{\mathbb{E}}%

\newcommand{\Var}{\mathrm{Var}}

\renewcommand{\P}{\mathbb{P}}
\newcommand{\R}{\mathbb{R}}

 \def \calE {\mathcal E}

 \def \calF {\mathcal F}

 \def \calL {\mathcal L}
 
 \def \calN {\mathcal N}

 \def \calS {\mathcal S}

\begin{document}

\twocolumn[
  \icmltitle{Anti-causal domain generalization: Leveraging unlabeled data}

\icmlnoneaffiliation{noaffil}

  \begin{icmlauthorlist}
    \icmlauthor{Sorawit Saengkyongam}{aapl}
    \icmlauthor{Juan L. Gamella}{noaffil}
    \icmlauthor{Andrew C.~Miller}{aapl}
    \icmlauthor{Jonas Peters}{eth}
    \icmlauthor{Nicolai Meinshausen}{eth}
    \icmlauthor{Christina Heinze-Deml}{aapl}
  \end{icmlauthorlist}

  \icmlaffiliation{aapl}{Apple}
  \icmlaffiliation{eth}{ETH Zürich}
  \icmlaffiliation{noaffil}{none}

  \icmlcorrespondingauthor{Sorawit Saengkyongam}{s.saengkyongam@apple.com}

  \icmlkeywords{Machine Learning, Causality, Invariance, Domain Generalization, Distributional Shifts}

  \vskip 0.3in
]

\printAffiliationsAndNotice{}  %

\begin{abstract}
The problem of domain generalization concerns learning predictive models that are robust to distribution shifts when deployed in new, previously unseen environments. Existing methods typically require labeled data from multiple training environments, limiting their applicability when labeled data are scarce. In this work, we study domain generalization in an anti-causal setting, where the outcome causes the observed covariates. Under this structure, environment perturbations that affect the covariates do not propagate to the outcome, which motivates regularizing the model's sensitivity to these perturbations. Crucially, estimating these perturbation directions does not require labels, enabling us to leverage unlabeled data from multiple environments. We propose two methods that penalize the model's sensitivity to variations in the mean and covariance of the covariates across environments, respectively, and prove that these methods have worst-case optimality guarantees under certain classes of environments. Finally, we demonstrate the empirical performance of our approach on a controlled physical system and a physiological signal dataset.
\end{abstract}

\section{Introduction}
Machine learning models are often trained on data from a limited set of environments and subsequently deployed in new, previously unseen environments.
A central challenge in this setting is domain generalization: learning predictive models that perform well not only on the training environments but also on novel test environments that may differ from those seen during training \citep{blanchard2011generalizing, muandet2013domain}. This challenge is particularly acute in high-stakes applications, such as healthcare, where distribution shifts between hospitals, patient populations, or measurement devices can significantly impact model performance \citep{subbaswamy2020development, degrave2021ai}.

One prominent approach to domain generalization leverages the framework of structural causal models \citep[SCMs;][]{Pearl2009, peters2016causal} to characterize the mechanisms by which distributions shift across environments. Under this framework, one can identify predictors that remain invariant, and hence robust, across a class of environments implied by the underlying structural causal model. This perspective has led to a variety of methods that exploit invariance for robust prediction \citep{ rojas2018invariant, magliacane2018domain, Arjovsky2019, heinze2021conditional, rothenhausler2021anchor, pfister2021stabilizing, saengkyongam2022exploiting, shen2025causality}. However, existing methods typically require labeled data from multiple environments to estimate invariance properties, which limits their applicability when labeled data are scarce or expensive to obtain. 

In many practical settings, unlabeled data are more abundant than labeled data. 
This motivates studying domain generalization in a type of semi-supervised setting: learning robust models using labeled data from only a small number of environments while leveraging unlabeled data from many others. 
A key question we address is under which assumptions unlabeled data can provide useful information about the structure of distribution shifts, even without outcome labels. This work demonstrates that this is possible
under an anti-causal learning setting, where the outcome $Y$ causes the observed predictors $X$. 
Anti-causal structures arise naturally in many applications. In healthcare, a patient's underlying physiological state (the outcome) often causes the observable measurements (the predictors). 
In speech recognition, the spoken content (the outcome) causes the observed audio signal (the predictors).

In an anti-causal setting, environment perturbations that only affect the predictors $X$ do not propagate to the outcome $Y$. Nevertheless, such perturbations can induce both covariate and concept shifts (see Section~\ref{sec:formulation}), harming predictive performance in unseen environments. 
To mitigate these distributional shifts, we develop regularization strategies that penalize sensitivity to environment perturbations.  
Crucially, the perturbation directions can be estimated solely from the marginal distribution of $X$, without requiring labels.
We provide theoretical guarantees showing that our regularized estimators are optimal for worst-case risk over a class of environments characterized by the directions of variation in the unlabeled data,
with the degree of extrapolation controlled by the regularization strength.

\paragraph{Contributions.} Our main contributions are as follows. 

\begin{enumerate}
    \item We formalize an anti-causal domain generalization framework in a semi-supervised setting, where labeled data are available from a few environments, and unlabeled data are available from many others (Section~\ref{sec:formulation}).
    \item We propose two regularization strategies, Mean-based Invariant Regularization (MIR) and Variance-based Invariant Regularization (VIR), that exploit distributional variations in the unlabeled data to encourage robustness. We provide theoretical guarantees, showing that the regularized estimators are optimal in terms of worst-case risk over certain classes of environments (Section~\ref{sec:method}).
    \item We evaluate our methods on two real-world datasets that naturally exhibit an anti-causal structure: a controlled physical system and a physiological signal dataset (Section~\ref{sec:experiment}).
\end{enumerate}

\section{Related Work}\label{sec:related_work}

Our work builds on the literature that exploits invariance in heterogeneous data to address domain generalization; 
see \citet{buhlmann2020invariance} for an overview. The concept of invariant prediction \citep{peters2016causal} has been employed to identify robust (or stable) predictors across various settings \citep{rojas2018invariant, HeinzeDeml2017, magliacane2018domain, pfister2021stabilizing, heinze2021conditional, saengkyongam2021invariant, saengkyongam2024effect}. Several works have extended this idea from selecting invariant predictors to incorporating invariance as a regularization term \citep{Arjovsky2019, heinze2021conditional, rothenhausler2021anchor, saengkyongam2022exploiting, shen2025causality}. These methods leverage causal frameworks to establish theoretical guarantees for distributional robustness within certain classes of distributions implied by SCMs.

To the best of our knowledge, none of the aforementioned works considers extracting invariance from unlabeled data. Our work utilizes assumptions from anti-causal structures, which allow us to define an invariance regularization that depends solely on unlabeled observations.
While \citet{scholkopf2012causal} discuss the implications of anti-causal learning in a single-domain semi-supervised setting, we extend these ideas to semi-supervised domain generalization. Anti-causal structures were also considered in \citet{heinze2021conditional,  veitch2021counterfactual, pmlr-v151-makar22a, jiang2022invariant, eastwood2023spuriosity} for domain generalization, but not within a semi-supervised setting.
Our theoretical guarantees are closely related to those in \citet{rothenhausler2021anchor} and \citet{shen2025causality}; however, ours differ by explicitly leveraging anti-causal assumptions to derive guarantees from invariance regularization that uses only unlabeled data.

Another line of research focuses on explicitly optimizing worst-case group performance \citep{MeinshausenMiniMax, buhlmann2015magging,  hu2018does, sagawa2020distributionally, freni2025maximum}, 
a framework often known as Group Distributionally Robust Optimization (GroupDRO). In GroupDRO, robustness is typically defined with respect to test distributions that are convex combinations of training distributions, which limits the model's ability to extrapolate beyond the training support. In contrast, our method considers test distributions that may lie outside the convex hull of the training distributions. Specifically, we define plausible directions of extrapolation based on the directions of variation observed in the unlabeled data. While \citet{krueger2021out} consider extrapolation beyond convex combinations by optimizing the worst-case affine combination of training risks, the directions along which the model can extrapolate from the training distributions are, to our knowledge, less well-characterized. 

\section{Anti-causal domain generalization in a semi-supervised setting}\label{sec:formulation}

This section formalizes the setting and goal of our work.
\begin{setting}[Multi-environment anti-causal learning]\label{setting:anti-causal}
Let $\calE$ be a collection of environments, we consider the following class of structural causal models indexed by $e \in \calE$: \vspace{0.3cm} \\
\begin{minipage}{0.49\columnwidth}
\begin{equation}\label{eq:scm_e}
\calS(e):
\begin{cases}
U \coloneqq \varepsilon_U \\
Y \coloneqq g_0(U, \varepsilon_Y) \\
X \coloneqq f_0(Y, U, \varepsilon_X) + \varepsilon_e,
\end{cases}
\end{equation}
\end{minipage}%
\quad
\begin{minipage}{0.45\columnwidth}
\begin{tikzpicture}[scale=1.0, node distance=1.0cm, roundnode/.style={circle, draw, inner sep=3pt,minimum size=7mm}, squarenode/.style={rectangle, draw, inner sep=1pt, minimum size=6mm}] 
  \node[squarenode] (E) at (-2, 0.8){$e$};
  \node[roundnode] (X) at (-1, 0){$X$};
  \node[roundnode] (Y) at (1, 0){$Y$};
  \node[roundnode][fill=black!25] (U) at (0, 0.8){$U$};
  \tikzstyle{EdgeStyle}=[line width=1, -Latex]
  \Edge[label=$f_0$](Y)(X)
  \Edge[labelstyle={pos=0.42}](E)(X)
  \tikzstyle{EdgeStyle}=[bend left=20, line width=1, -Latex, dashed]
  \Edge[](U)(Y)
  \tikzstyle{EdgeStyle}=[bend right=20, line width=1, -Latex, dashed]
  \Edge[](U)(X)
\end{tikzpicture}
\end{minipage} \vspace{0.2cm}\\
where $X \in \R^d$ denotes the predictors, $U \in \R^q$ denotes the unobserved confounders, $Y \in \R$ denotes the outcome, $(\varepsilon_U, \varepsilon_Y, \varepsilon_X)$ are jointly independent noise variables, and for each $e \in \calE$, the environment perturbation $\varepsilon_e \sim \P^e_{\varepsilon}$ satisfies $\varepsilon_e \ci (\varepsilon_U, \varepsilon_Y, \varepsilon_X)$. Without loss of generality, we assume that $f_0(Y, U, \varepsilon_X)$ and $g_0(U, \varepsilon_Y)$ have zero mean.
\end{setting}

\begin{remark}\label{remark:setting}
\begin{enumerate}[label=(\roman*)]
    \item We do not impose structural minimality on the SCM~\eqref{eq:scm_e}: $f_0$ may or may not depend on $Y$. Hence, the SCM~\eqref{eq:scm_e} subsumes the case where the $X$--$Y$ dependence is driven solely by the latent state $U$ without a direct edge from $Y$ to $X$. For this reason, we adopt a broader notion of ``anti-causal'' to refer to settings in which the outcome $Y$ is not a descendant of the predictors $X$ (as opposed to $Y$ being the direct cause of $X$).
    \item Although the SCM~\eqref{eq:scm_e} imposes certain restrictions, such as the absence of a direct effect from $e$ to $Y$, this model class can still capture non-trivial distribution shifts. In particular, it allows for both covariate shift and concept shift simultaneously: for some $e_1, e_2 \in \calE$, we may have both $\P^{e_1}_X \neq \P^{e_2}_X$ and $\P^{e_1}_{Y \mid X} \neq \P^{e_2}_{Y \mid X}$.
\end{enumerate}
\end{remark}

\paragraph{Notation.} For each environment $e \in \calE$, we denote by $\P_{X,Y}^e$ the distribution induced by $\calS(e)$, and by $\EX^e$ and $\Var^e$ the corresponding expectation and variance.

\subsection{Learning setting}\label{sec:learning_setting}

We are given observations from $p$ training environments $\calE^{\tr} \coloneqq \{e_1, \dots, e_p\}\subseteq \calE$, where we have 
\begin{itemize}
    \item Labeled data from a small set of environments: $D^{\tr}_{X,Y} = \{(X_i, Y_i, E_i)\}_{i=1}^{k}$ such that $(X_i, Y_i) \sim \P^{E_i}_{X,Y}$, where $E_i \in \calL \subset \calE^{\tr}$, 
    with $\calL$ denoting the set of labeled environments, and $k$ is the total number of labeled observations. We assume $\{E_1, \ldots, E_k\} = \calL$.
    \item Unlabeled data from all training environments: 
    $D^{\tr}_X = \{(X_j, E_j)\}_{j=1}^{n}$ such that $X_j \sim \P^{E_j}_{X}$, where $E_j \in \calE^{\tr}$ 
    and $n$ is the total number of unlabeled observations, with $n \geq k$ (the first $k$ observations are the $X_i$ and $E_i$ in $D^{\tr}_{X,Y}$).
    We assume\footnote{This assumption ensures that we observe at least one observation from each environment in $\calE^{\tr}$.} $\{E_1, \ldots, E_n\} = \calE^{\tr}$.
\end{itemize}

Let $\calF$ be a function class of interest. Our goal is to learn a predictive function $f^{\tr} \in \calF$ using $D^{\tr}_{X,Y}$ and $D^{\tr}_X$ such that it generalizes well to an unseen test environment $e_{\tst} \in \calE$. Since we do not have additional information on the test environment $e_{\tst}$, it can be any environment in $\calE$, so possibly $e_{\tst} \notin \calE^{\tr}$. We consider the worst-case risk as the objective:
\begin{equation}\label{eq:robust_objective}
    \sup_{e \in \calE} \EX^e[\ell(f^{\tr}(X), Y)],
\end{equation}
where $\ell: \R \times \R \rightarrow \R$ is a given loss function.

As discussed in Section~\ref{sec:related_work}, the objective~\eqref{eq:robust_objective} depends crucially on the class of environments $\calE$ under consideration. 
Our approach characterizes the class $\calE$ using the directions of variation observed in the unlabeled data. 
This allows the model to extrapolate along such directions, with the degree of extrapolation controlled by a regularization parameter. We now introduce our approach and formalize these ideas.

\section{Data-driven robust regression}\label{sec:method}

The key idea of our approach is to leverage the invariance encoded in the assumed SCM~\eqref{eq:scm_e}: the environment perturbations $\varepsilon_e$ are independent of the outcome $Y$, i.e.,
\begin{equation}\label{eq:invariance}
    \forall e \in \mathcal{E}: \varepsilon_e \perp\!\!\!\perp Y.
\end{equation}
Intuitively, if we can characterize the structure of the variations in $\varepsilon_e$ across environments, we can penalize models along these directions, since (i) under Setting~\ref{setting:anti-causal}\footnote{Setting~\ref{setting:anti-causal} can be relaxed to allow the environment $E$ to influence $Y$ (or $U$) through an observed covariates $W$. Concretely, suppose we additionally observe $W$ such that $W$ blocks all paths from $E$ to $Y$. Then the invariance \eqref{eq:invariance} still holds conditionally on $W$, and the methods proposed in Section~\ref{sec:method} can be extended. We discuss this extension in more detail in Appendix~\ref{app:mediator_extension}.}, such variations are the  source of distributional shifts across environments, and (ii), by~\eqref{eq:invariance}, these variations carry no information about the outcome.
Although $\varepsilon_e$ is not directly observable, we show in Theorems~\ref{thm:MIR}~and~\ref{thm:VIR} that it is possible to learn provably robust models by exploiting distributional shifts estimated from unlabeled data. To this end, we develop regularization strategies that penalize models according to their sensitivity to these inferred shifts to promote robustness to environment shifts.

In this section, we focus on the setting where $\calF$ consists of linear functions and the loss $\ell$ is the squared error. However, these are not fundamental limitations; we discuss extensions to other loss functions and nonlinear models in Section~\ref{sec:extension}.

For the remainder of this section, let $\calE$ be a collection of environments such that $\{\P^e_{\varepsilon}\}_{e \in \calE}$ contains all probability distributions over $\R^d$. We define $\EX^{\tr}[\cdot] \coloneqq \frac{1}{|\calL|} \sum_{e \in \calL} \EX^{e}[\cdot]$ as the average expectation over labeled training environments.

\subsection{Mean-based Invariant Regularization (MIR)}\label{sec:MIR}

Our first regularization strategy targets variations in the means of $\varepsilon_e$ across environments. Define the matrix of environment-specific covariate means as
\begin{equation*}
    K \coloneqq \begin{bmatrix} \EX^{e_1}[X] & \cdots & \EX^{e_p}[X] \end{bmatrix} \in \R^{d \times p},
\end{equation*}
which can be estimated solely from the unlabeled data. We consider the following regularized regression coefficients:
\begin{equation}\label{eq:MIR_objective}
    \beta^{\text{MIR}}_{\gamma} = \argmin_{\beta \in \R^d} \EX^{\tr}[(Y - \beta^\top X)^2] + \gamma \beta^\top \Var(K) \beta,
\end{equation}
where $\Var(K) \coloneqq \frac{1}{p} K K^\top - \big(\frac{1}{p} K \mathbf{1}_p\big)\big(\frac{1}{p} K \mathbf{1}_p\big)^\top$ denotes the covariance matrix of the environment means and $\mathbf{1}_p$ is the vector of ones.

In Setting~\ref{setting:anti-causal}, $f_0(Y, U, \varepsilon_X)$ has zero mean, and hence $\EX^{e}[X] = \EX^e[\varepsilon_e]$ for all environments $e \in \calE$. The matrix $K$ therefore equals 
\begin{equation*}
    K = \mu^{\tr}_{\varepsilon} \coloneqq \begin{bmatrix} \EX[\varepsilon_{e_1}] & \cdots & \EX[\varepsilon_{e_p}] \end{bmatrix}.
\end{equation*}
Thus, the MIR regularization term $\beta^\top \Var(K) \beta$ penalizes sensitivity to mean shifts in environment perturbations $\varepsilon_e$.

In Theorem~\ref{thm:MIR} below, we establish a dual characterization of~\eqref{eq:MIR_objective} in terms of distributional robustness: the model coefficients that minimize the regularized objective~\eqref{eq:MIR_objective} are also optimal for the worst-case risk under a class of mean-shift perturbations, with the size of this class controlled by the regularization parameter $\gamma$. (All proofs are in Appendix~\ref{app:proofs}.)

\begin{theorem}[MIR Robustness]\label{thm:MIR}
Define
\[
\Delta^{\diamond}_{\gamma} \coloneqq \left\{ A \in \R^{d \times d} : 0 \preceq A \preceq \gamma \Var(\mu^{\tr}_{\varepsilon}) \right\}.
\]
Under Setting~\ref{setting:anti-causal}, we have
\begin{align*}
    &\beta^{\emph{MIR}}_{\gamma}
    \in \argmin_{\beta \in \R^d} \sup_{e \in \calE^{\diamond}_{\gamma}} \EX^{e}[(Y - \beta^\top X)^2], \text{ where} \\
     \calE^{\diamond}_{\gamma} &\coloneqq \{e_{\tst} \in \calE : \exists A \in \Delta^{\diamond}_{\gamma} \text{ \emph{s.t.} } \\ 
     &\hspace{2cm} \EX^{e_{\tst}}[\varepsilon_e \varepsilon_e^\top] = \EX^{\tr}[\varepsilon_{e} \varepsilon_{e}^\top] + A \}.
\end{align*}
\end{theorem}

\begin{figure}[!t]
    \centering
    \includegraphics[width=\columnwidth]{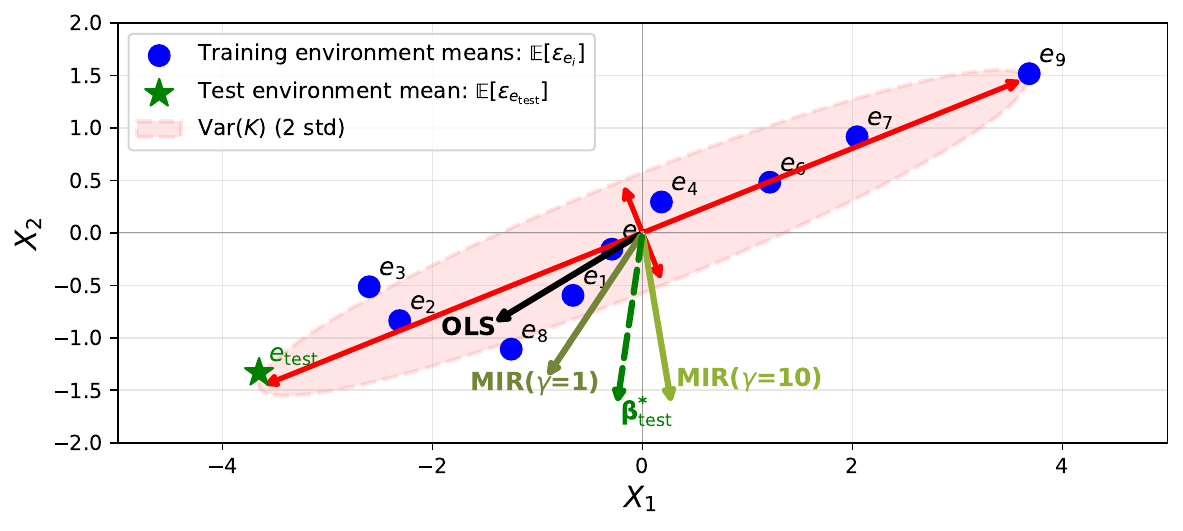} 
    \caption{Illustration of MIR. Blue points represent the training environment perturbation means $\EX[\varepsilon_{e_i}]$, and the green star represents the test environment mean. The red ellipse represents the covariance structure of the regularization matrix $\Var(K)$, with red arrows indicating its (scaled) eigenvectors.
    The annotated arrows show the directions of the OLS and MIR solutions with different regularization strengths $\gamma$, as well as the optimal solution $\beta^{*}_{\mathrm{test}}$ for the test environment $e_{\tst}$.}
    \label{fig:mir_illustration}
    \vspace{-1em}
\end{figure}

Figure~\ref{fig:mir_illustration} illustrates MIR using a simulated example with $p = 10$ environments and 2-dimensional covariates, where one environment is held out as the test environment. The environment perturbation means $\{\EX[\varepsilon_{e_i}]\}_{i=1}^{10}$ are generated from a Gaussian distribution with a covariance structure such that one eigenvalue is much larger than the other. The figure shows the direction of the ordinary least squares (OLS) solution, the MIR solutions for two regularization strengths $\gamma \in \{1, 10\}$, and the optimal solution $\beta^{*}_{\mathrm{test}}$ for the test environment $e_{\tst}$. The regularization strength $\gamma$ controls the tradeoff between the OLS solution ($\gamma = 0$) and the direction of least variability. The optimal choice of $\gamma$ for a given test environment depends on how far the test environment mean is from the training centroid:
guarding against
larger shifts requires stronger regularization. In this example, $\Var(K)$ is full rank, so as $\gamma$ increases, the solution rotates away from the direction of high variability (the first eigenvector) toward the direction of least variance (the second eigenvector), and then shrinks toward zero as $\gamma \to \infty$. More generally, when $\Var(K)$ is not full rank, the solution shrinks toward a vector in the null space of $\Var(K)$ as $\gamma \to \infty$; this null space corresponds to the subspace of directions that are invariant across training environments.

\subsection{Variance-based Invariant Regularization (VIR)}\label{sec:VIR}
Our second strategy targets variations in the covariances 
of $\varepsilon_e$ across environments. 
For each $e \in \calE$, define the covariance matrix $G^{e}_X \coloneqq \Var^{e}(X)$ and the average covariance $\bar{G}_X \coloneqq \tfrac{1}{p} \sum_{j=1}^p G^{e_j}_X$. We consider the following objective:
\begin{align*}\label{eq:VIR_objective}
    \beta^{\text{VIR}}_{\gamma} = 
    \argmin_{\beta \in \R^d}& \EX^{\tr}[(Y - \beta^\top X)^2] \\ 
    &+ \gamma \beta^\top \left(\frac{1}{p}\sum_{i=1}^p (G^{e_i}_X - \bar{G}_X)^2\right) \beta. \numberthis
\end{align*}
This regularization penalizes models that are sensitive to variations in the covariances observed in the unlabeled data.

Analogously to the MIR case, this objective admits a distributionally robust interpretation: robustness is defined with respect to shifts in the covariances of $\varepsilon_{e}$, with the size of the perturbation set controlled by $\gamma$.

\begin{theorem}[VIR Robustness]\label{thm:VIR}
For each $e \in \calE$, let $G^{e}_{\varepsilon} \coloneqq \Var^e(\varepsilon_{e})$ and $\bar{G}_{\varepsilon} \coloneqq \tfrac{1}{p} \sum_{j = 1}^p G^{e_j}_{\varepsilon}$. Define 
\[
\Delta^\dagger_{\gamma} \coloneqq \left\{ A \in \R^{d \times d} :  0 \preceq  A \preceq \gamma \tfrac{1}{p} \sum_{i=1}^p (G^{e_i}_{\varepsilon} - \bar{G}_{\varepsilon})^2 \right\}.
\]
Under Setting~\ref{setting:anti-causal}, we have
\begin{align*}
    \beta^{\emph{VIR}}_{\gamma} 
    &\in \argmin_{\beta \in \R^d} \sup_{e \in \calE^\dagger_{\gamma}} \EX^{e}[(Y - \beta^\top X)^2], \text{ where} \\
    \calE^\dagger_{\gamma} &\coloneqq \{e_{\tst} \in \calE : \exists A \in \Delta^\dagger_{\gamma} \text{ \emph{s.t.}\ } \\ &\hspace{2cm}\EX^{e_{\tst}}[\varepsilon_e \varepsilon_e^\top] = \EX^{\tr}[\varepsilon_{e} \varepsilon_{e}^\top] + A \}.
\end{align*}
\end{theorem}

To provide further intuition for VIR, consider the case where the covariance matrices of the environment perturbations, $\{G^{e_i}_{\varepsilon}\}_{i=1}^{p}$, share a common eigenbasis, i.e., there exists an orthogonal matrix $Q \in \R^{d \times d}$ such that, for all $i \in \{1, \ldots, p\}$: $G^{e_i}_{\varepsilon} = Q \Lambda^{e_i} Q^\top$, where $\Lambda^{e_i} = \diag(\lambda^{e_i}_1, \ldots, \lambda^{e_i}_d)$ is the diagonal matrix of eigenvalues of $G^{e_i}_{\varepsilon}$. Define the per-coordinate variance of eigenvalues across environments as
\[
\sigma^2_j \coloneqq \frac{1}{p} \sum_{i=1}^{p} \left( \lambda^{e_i}_j - \bar{\lambda}_j \right)^2, \quad \text{where } \bar{\lambda}_j \coloneqq \frac{1}{p} \sum_{i=1}^{p} \lambda^{e_i}_j.
\]
In Appendix~\ref{app:VIR_shared_eigenbasis}, we show that the VIR regularization term simplifies to
\begin{equation}
\beta^\top \left( \frac{1}{p} \sum_{i=1}^{p} (G^{e_i}_X - \bar{G}_X)^2 \right) \beta = \sum_{j=1}^{d} \sigma^2_j (\tilde{\beta}_j)^2,
\end{equation}
where $\tilde{\beta} = Q^\top \beta$ denotes the model coefficients projected onto the shared eigenspace. Recall that the eigenvalues $\lambda^{e_i}_j$ represent the variance along the $j$-th eigenvector direction in environment $e_i$. This simplification provides a clear interpretation: VIR applies stronger regularization to directions along which the variances differ more across environments, thereby encouraging the model to rely on directions that are more stable. 

\begin{remark}\label{remark:higher_order}
\begin{enumerate}[label=(\roman*)]
    \item Appendix~\ref{app:VIR_alternative} discusses an alternative formulation of VIR that directly penalizes changes in the variance of predictions $\Var(\beta^\top X)$ across environments. Although this may sound intuitive, it is vulnerable to a cancellation issue: there may exist directions $\beta$ that cancel out the observed training shifts, yielding zero penalty even though the direction is sensitive to unseen shifts. Our VIR avoids this issue; see the appendix for such an example. 
    \item In practice, environment perturbations may affect both the means and covariances. We can combine both regularization terms in a single objective, as discussed in Appendix~\ref{app:MIR_VIR}. 
    \item Our use of causal structure (see Setting~\ref{setting:anti-causal}) differs from its use in causal identification: rather than yielding an estimator whose validity hinges on the assumed graph, the causal assumptions motivate when the proposed regularization should improve generalization. In practice, the regularization strength $\gamma$ is selected by nested leave-one-environment-out cross-validation on the training environments. When the assumptions are substantially violated, the regularization no longer improves out-of-environment risk, and the cross-validation procedure would select $\gamma \approx 0$, falling back to the unregularized estimator. We illustrate this empirically in Section~\ref{sec:violated_assumptions}.
\end{enumerate}
\end{remark}

\section{Finite-sample estimators}\label{sec:estimators}

In this section, we discuss finite-sample estimators for the regularized regressions introduced in Section~\ref{sec:method}. We first present the population-level solutions and introduce plug-in estimators based on sample quantities. We then discuss consistency of the resulting estimators.

\subsection{Population-level solutions}

Consider the MIR and VIR objectives~\eqref{eq:MIR_objective} and~\eqref{eq:VIR_objective} with squared error loss. Both objectives are quadratic in $\beta$ and admit closed-form solutions.
The population-level minimizer takes the form
\begin{equation}\label{eq:population_solution}
    \beta_{\gamma} = \left( \EX^{\tr}[X X^\top] + \gamma H \right)^{-1} \EX^{\tr}[X Y],
\end{equation}
where $H \in \R^{d \times d}$ is a positive semi-definite matrix that encodes the directions of variation of environment perturbations $\varepsilon_e$. Specifically, for MIR, we have $H_{\text{MIR}} = \Var(K)$, and for VIR, $H_{\text{VIR}} = \frac{1}{p} \sum_{i=1}^{p} \left( G^{e_i}_X - \bar{G}_X \right)^2$.

\subsection{Plug-in estimators}

Given labeled data $D^{\tr}_{X,Y}$ and unlabeled data $D^{\tr}_X$ (as defined in Section~\ref{sec:learning_setting}), we construct plug-in estimators by replacing population quantities with their sample counterparts. Let $\mathbf{X} \in \R^{k \times d}$ denote the design matrix with rows $X_i^\top$ and $\mathbf{Y} \in \R^{k}$ denote the outcome vector, both constructed from the labeled data $D^{\tr}_{X,Y}$. The plug-in estimator is
\begin{equation}\label{eq:closed_form}
    \hat{\beta}_{\gamma} = \left( \frac{1}{k} \mathbf{X}^\top \mathbf{X} + \gamma \hat{H} \right)^{-1} \frac{1}{k} \mathbf{X}^\top \mathbf{Y},
\end{equation}
where $\hat{H} \in \R^{d \times d}$ is estimated from the unlabeled data $D^{\tr}_X$.

For MIR, we use
\begin{equation}\label{eq:H_MIR}
    \hat{H}_{\text{MIR}} \coloneqq\frac{1}{p} \hat{K} \hat{K}^\top - \big(\frac{1}{p} \hat{K} \mathbf{1}_p\big)\big(\frac{1}{p} \hat{K} \mathbf{1}_p\big)^\top,
\end{equation}
where $\hat{K} \coloneqq \begin{bmatrix} \hat{\mu}^{e_1}_X & \cdots & \hat{\mu}^{e_p}_X \end{bmatrix} \in \R^{d \times p}$ and $\hat{\mu}^{e_i}_X \coloneqq \frac{1}{n_i} \sum_{j \in \calN_i} X_j$ is the sample mean of covariates in environment $e_i$, with $\calN_i \coloneqq \{j \in \{1,\dots,n\} | E_j = e_i\}$ and $n_i \coloneqq |\calN_i|$.

For VIR, we use
\begin{equation}\label{eq:H_VIR}
    \hat{H}_{\text{VIR}} \coloneqq \frac{1}{p} \sum_{i=1}^{p} \left( \hat{G}^{e_i}_X - \hat{\bar{G}}_X \right)^2,
\end{equation}
where $\hat{G}^{e_i}_X \coloneqq \frac{1}{n_i} \sum_{j \in \calN_i} (X_j - \hat{\mu}^{e_i}_X) (X_j - \hat{\mu}^{e_i}_X)^\top$ is the sample covariance matrix in environment $e_i$ and $\hat{\bar{G}}_X \coloneqq \frac{1}{p} \sum_{i=1}^{p} \hat{G}^{e_i}_X$ is the average sample covariance across environments.

The following proposition establishes that the plug-in estimators are consistent for the population-level solutions.

\begin{proposition}[Consistency]\label{prop:consistency}
Assume that $\EX^e[\|X\|_2^2] < \infty$ for all $e \in \calE^{\tr}$ and $\EX^e[Y^2] < \infty$ for all $e \in \calL$, and that $\EX^{\tr}[X X^\top] + \gamma H$ is non-singular. Assume further that the labeled data are sampled in a balanced manner\footnote{The balanced-sampling assumption can be dropped by redefining $\EX^{\tr}[\cdot]$ as the corresponding weighted average. We adopt the balanced version here for clarity of presentation.}, i.e., $k_e / k \to 1/|\calL|$ for each $e \in \calL$ as $k \to \infty$, where $k_e$ is the number of labeled observations from environment $e$. If $k \to \infty$ and $n_i \to \infty$ for all $i \in \{1, \dots, p\}$, then
\begin{equation}
    \hat{\beta}_{\gamma} \xrightarrow{p} \beta_{\gamma}.
\end{equation}
\end{proposition}

\section{Extension: Different loss functions and nonlinear models}\label{sec:extension}

In the previous sections, we considered the setting where $\calF$ is the class of linear functions and the loss $\ell$ is the squared error. This section discusses extensions of our approach to other loss functions and to nonlinear models.

\subsection{Other loss functions} \label{sec:other_loss}
It is straight-forward to apply the proposed methods to other loss functions: the regularization terms (both MIR and VIR) depend solely on the unlabeled data, so the only change to the optimization problem for learning is the addition of a quadratic term.
For example, 
for the MIR objective we get
\begin{equation}\label{eq:MIR_other_losses}
    \argmin_{\beta \in \R^d} \EX^{\tr}[\ell(\beta^\top X, Y)] + \gamma \beta^\top \Var(K) \beta,
\end{equation}
where $K = \begin{bmatrix} \EX^{e_1}[X] & \cdots & \EX^{e_p}[X] \end{bmatrix}$ is defined as before. Although the precise robustness guarantee differs from Theorem~\ref{thm:MIR}, the intuition remains the same: we penalize $\beta$ in directions that are sensitive to variation in the means of $\varepsilon_e$.
VIR can be applied analogously by replacing the MIR regularization term with that of VIR defined in~\eqref{eq:VIR_objective}.

We view this flexibility as another advantage of our approach compared to methods that rely on residual invariance, such as Anchor regression~\citep{rothenhausler2021anchor}. In particular, extending Anchor regression to other losses requires defining new notions of residuals with more involved regularization terms~\citep{kook2022distributional}.

\subsection{Nonlinear models}\label{sec:nonlinear}

To apply our approach to high-dimensional and/or unstructured data, we adopt a representation learning perspective: we first transform $X$ via a nonlinear map $\phi$ and then apply a linear predictor on top of the learned representation $\phi(X)$. The regularization is then applied to the linear predictor. This corresponds to modeling environment perturbations in a latent space rather than in the observed covariate space, which is particularly relevant when $X$ consists of unstructured data such as sensors, images, or text, where environment perturbations may be more meaningfully captured in a learned representation.
Similar considerations have been discussed in, e.g., \citet{rosenfeld2022domain, saengkyongam2024identifying, vsola2025causality, von2025representation}.

Concretely, for a function class $\Phi$, we consider the MIR objective
\begin{equation}\label{eq:MIR_nonlinear}
    \argmin_{\phi \in \Phi, \beta \in \R^{d'}} \EX^{\tr}[\ell(\beta^\top \phi(X), Y)] + \gamma \beta^\top \Var(K_{\phi}) \beta,
\end{equation}
where $K_{\phi} \coloneqq \begin{bmatrix} \EX^{e_1}[\phi(X)] & \cdots & \EX^{e_p}[\phi(X)] \end{bmatrix} \in \R^{d' \times p}$ and $d'$ is the dimension of the representation. Here, the regularization encourages the linear predictor $\beta$ to be insensitive to environment perturbations in the representation space. VIR can be extended analogously by computing the variance of covariances in the representation space, as in~\eqref{eq:VIR_objective}.

In many applications, the representation $\phi$ may be learned separately from the objective~\eqref{eq:MIR_nonlinear} rather than through joint optimization. Given the recent advancements in foundation models that provide meaningful representations across various data modalities, our regularization strategy can be naturally combined with these pretrained representations by fine-tuning  the linear head $\beta$ using the objective \eqref{eq:MIR_nonlinear} while keeping the representation $\phi$ fixed. We adopt this approach in our real-world experiments (see Section~\ref{sec:vital_db}).

\section{Experiments}\label{sec:experiment}
We evaluate our methods on two complementary real-world datasets that naturally exhibit anti-causal structure: a controlled physical system and a real-world biomedical application. Full details on experimental setups are provided in Appendix~\ref{app:light_tunnel}~and~\ref{app:vitaldb}.

\subsection{Causal Chamber: The Light Tunnel}
\label{sec:causal_chamber}

The Light Tunnel \citep{gamella2025causal} is a real, controllable, optical experiment designed for benchmarking machine learning methods. The tunnel consists of different light sources, polarizers and sensors that allow to measure and manipulate the physical variables of the system, while providing a causal ground truth of the effects between them.

\begin{figure}[!t]
    \centering
    \includegraphics[width=\columnwidth]{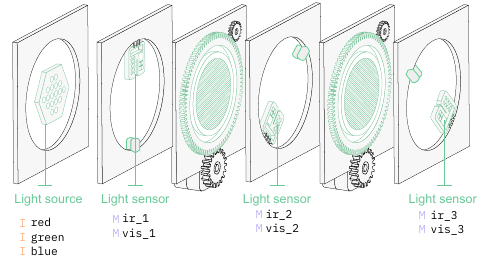} 
    \caption{
    Diagram of the light tunnel and the subset of variables used in our experiment: the outcomes or intervention targets \texttt{red}, \texttt{green} and \texttt{blue}, and the light-intensity measurements \texttt{ir\_1}, \texttt{vis\_1}, \ldots, \texttt{vis\_3} used as predictors. Figure adapted from \citet{gamella2025causal}, licensed under CC BY 4.0.
    }
    \label{fig:lt}
    \vspace{-1em}
\end{figure}

For our experiments, we focus on the main RGB light source and its effect on the different light-intensity sensors placed throughout the tunnel (Figure~\ref{fig:lt}). We select the brightness setting of one of the three color channels (red, green or blue) as the outcome $Y \in \R$ and use the measurements produced by the sensors as the predictors $X \in \R^6$.

We create six 
environments $e \in \calE$ by performing interventions on another color channel by shifting the uniform distribution from which their brightness settings are sampled.
By construction, these interventions induce only mean shifts in the covariate distribution, making this a suitable setting to evaluate MIR (VIR is evaluated in another experiment, see Section~\ref{sec:vital_db}).
We repeat the overall experiment six times by considering each outcome-intervention pair (e.g., outcome red with intervention on blue), resulting in a total of six experiments with six distinct environments each.

\paragraph{Baselines.} We compare the following methods:
(i) \texttt{Pooled-Ridge}: Ridge regression pooling data across all training environments.
(ii) \texttt{MIR}: Our mean-based invariant regularization, with the regularization parameter selected via leave-one-environment-out cross-validation on the training environments.
(iii) \texttt{MIR-Oracle}: Our method with the regularization parameter tuned on the test environment, serving as an upper bound on \texttt{MIR} performance.
(iv) \texttt{AnchorReg}: Anchor regression \citep{rothenhausler2021anchor}.
(v) \texttt{GroupDRO}: Group distributionally robust optimization \citep{sagawa2020distributionally}.

\paragraph{Evaluation setup.} For each outcome-intervention pair, we adopt a leave-one-environment-out evaluation protocol: one environment is held out for testing while models are trained on the remaining five environments. To evaluate the semi-supervised setting, we vary the number of training environments that have labeled data from 3 to 5,
while using unlabeled data from all training environments for \texttt{MIR}. Predictive performance is assessed using root mean squared error (RMSE) on the held-out environment (Figure~\ref{fig:chamber}). 

\begin{figure}[!t]
    \centering
    \includegraphics[width=\columnwidth]{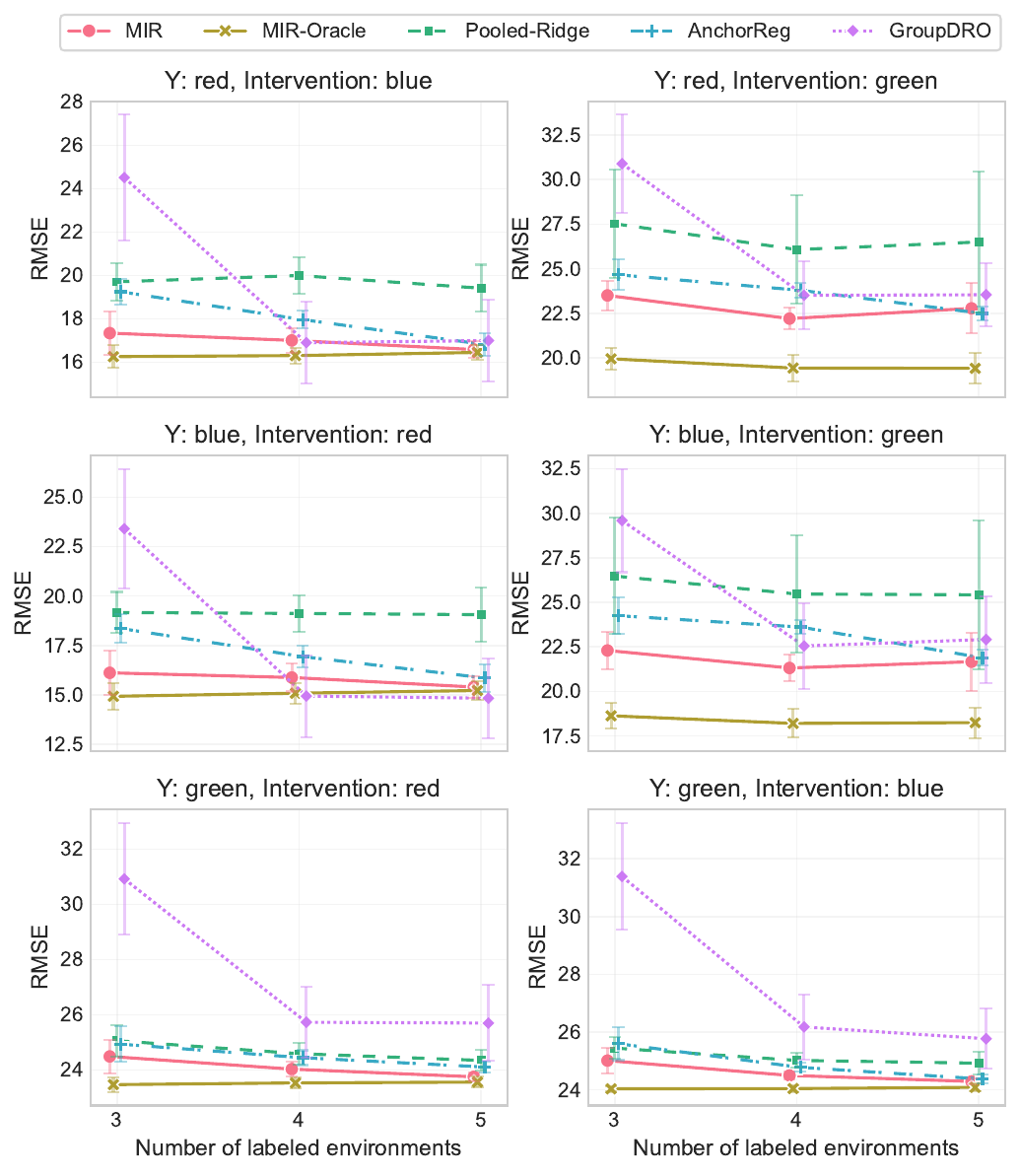} 
    \caption{
    Performance on the Light Tunnel dataset. We show average RMSE (with standard errors) for leave-one-environment-out cross-validation across all outcome-intervention combinations. The x-axis indicates the number of environments with labeled observations; \texttt{MIR} uses unlabeled data from all training environments regardless of this number. 
    }
    \label{fig:chamber}
\end{figure}

\paragraph{Results.} Figure~\ref{fig:chamber} shows the results across all outcome-intervention pairs. \texttt{Pooled-Ridge} performs poorly when generalizing to unseen environments, particularly when few labeled environments are available. This indicates that, in this experiment, standard empirical risk minimization is not sufficient to generalize to unseen environments. \texttt{MIR} achieves lower or comparable error to all baselines across all settings. The improvement over baselines is most pronounced when the number of labeled environments is small; this is the regime where 
leveraging unlabeled data is most valuable. 
Comparing \texttt{MIR} and \texttt{MIR-Oracle} suggests that leave-one-environment-out cross-validation often yields  regularization parameters that are close to being optimal but also that there are settings which carry potential for improved ways of choosing the regularization parameter.

\subsection{VitalDB: Stroke Volume Prediction from Arterial Waveforms}\label{sec:vital_db}

As a real-world application in the healthcare domain, we apply our approach to predicting stroke volume from arterial pressure waveform (APW) signals using data from the VitalDB dataset \citep{lee2022vitaldb}. Predicting stroke volume for unseen subjects is challenging due to pronounced between-subject variability that induces distribution shifts. This task exhibits a natural anti-causal structure: the stroke volume $Y$ causally influences the observed arterial pressure waveform $X$, while unobserved physiological factors $U$ may affect both.

\paragraph{Preprocessing.} Absolute prediction of hemodynamic parameters from arterial waveforms typically requires subject-specific calibration, such as thermodilution, to establish a baseline reference point \citep{saugel2021cardiac, manduchi2024leveraging}. Without such calibration, differences in baseline physiology across subjects (e.g., due to age, vascular stiffness, or body composition) may confound the relationship between waveform features and absolute hemodynamic values. In the context of Setting~\ref{setting:anti-causal}, these baseline differences correspond to direct shifts from $E$ to $U$ (or $Y$), which would violate the assumption that the environment does not affect $Y$ (through $U$). We therefore focus on predicting variations in stroke volume rather than absolute values. Concretely, we consider a subject-centered approach where both predictors and outcomes are mean-centered within each subject. To obtain absolute predictions, one can combine our predictions with a baseline reference point from standard calibration procedures \citep{saugel2021cardiac}. In this subject-centered setting, mean shifts are removed by construction, making our Variance-based Invariant Regularization (VIR) the appropriate method to address the remaining covariance shifts in $X$ across subjects.

\paragraph{Feature representation.} Following the formulation in Section~\ref{sec:nonlinear}, we extract 140-dimensional embeddings from arterial pressure waveforms using a pretrained convolutional neural network developed by \citet{manduchi2024leveraging, palumbo2025hybrid}. This model was trained using a neural posterior estimator \citep{lueckmann2017flexible} on synthetic data generated from the OpenBF cardiovascular simulator \citep{Benemerito_2024}, which models hemodynamics based on physiological parameters. The learned embeddings capture physiologically relevant features of the arterial pressure waveform without requiring labeled clinical data for training \citep{manduchi2024leveraging, palumbo2025hybrid}. This pretrained representation serves as the fixed nonlinear map $\phi$, and we apply all considered methods to the linear head on top of these embeddings. Concretely, we have:
\begin{itemize}
    \item $X \in \R^{140}$: APW embeddings (subject-centered)
    \item $Y \in \R$: Stroke volume variations (subject-centered)
    \item $e \in \calE$: Individual subjects ($n = 128$ subjects)
\end{itemize}

\paragraph{Baselines.} We consider the same baselines as in Section~\ref{sec:causal_chamber}, with the exception of \texttt{AnchorReg}: since the subject-centered preprocessing removes all mean shifts by construction, the Anchor regression regularization term vanishes. We replace it with \texttt{DRIG} (Distributional Robustness via Invariant Gradients)~\citep{shen2025causality}, an extension of Anchor regression that also tackles covariance shifts.

\paragraph{Evaluation setup.} We treat each subject as a distinct environment and adopt a leave-one-subject-out evaluation protocol: one subject is held out for testing while models are trained on the remaining subjects. To evaluate the semi-supervised setting, we vary the number of training subjects with labeled observations from 20 to 127, while using unlabeled data from all training subjects for VIR. Predictive performance is assessed on the held-out subject using normalized mean squared error (nMSE), defined as the mean squared error divided by the variance of the subject's outcome. This normalization ensures that subjects with different outcome variability contribute equally to the evaluation; otherwise, subjects with larger stroke volume variations would dominate the aggregate error. 
To assess robustness, we report the conditional value at risk (CVaR), computed as the average nMSE for subjects above a given quantile threshold of the nMSE distribution, where higher quantiles correspond to worse-performing subjects. Additionally, we report Spearman's correlation between predicted and true stroke volume variations to evaluate tracking performance (i.e., how well predictions follow the temporal variations within a subject regardless of magnitude).

\begin{figure}[!t]
    \centering
    \begin{subfigure}[b]{1.\columnwidth}
        \includegraphics[width=\columnwidth]{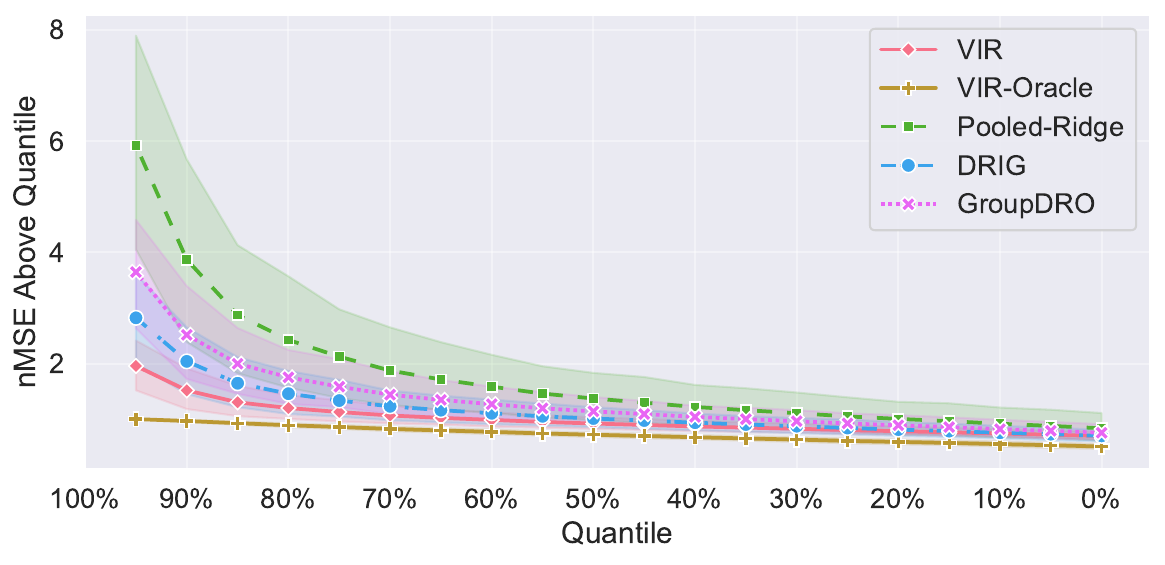}
        \subcaption{Robustness analysis}
        \label{fig:vitaldb_robustness}
    \end{subfigure}
    \begin{subfigure}[b]{1.\columnwidth}
        \includegraphics[width=\columnwidth]{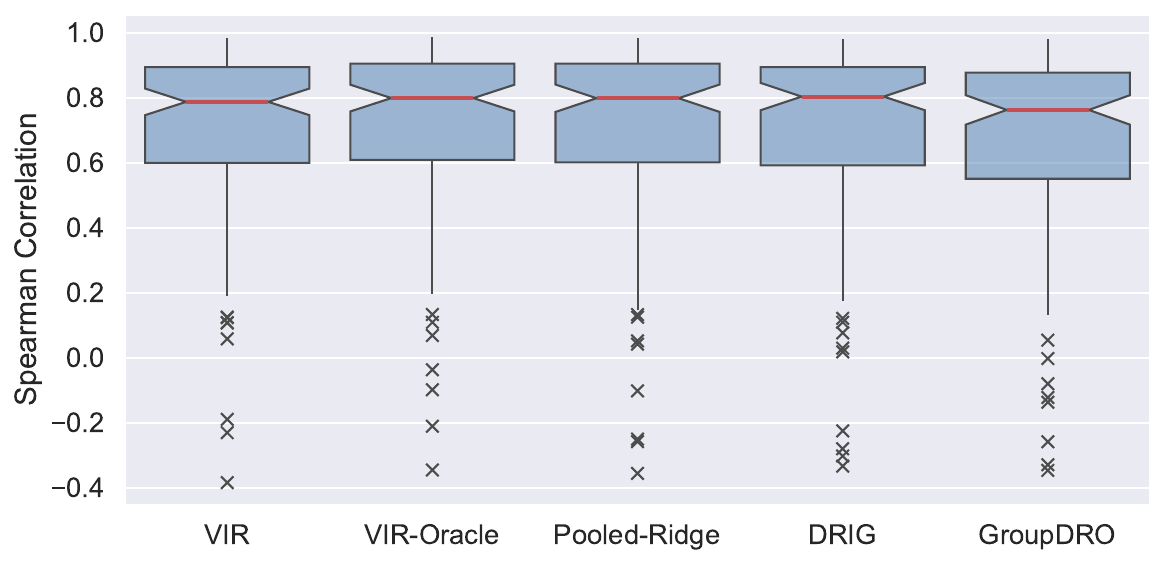}
        \subcaption{Tracking performance 
        }
        \label{fig:vitaldb_corr}
    \end{subfigure}
\caption{Performance on VitalDB dataset when all 128 training subjects have labeled data. (a) CVaR (average nMSE for subjects whose errors are above a given quantile) as a function of quantile threshold. (b) Distribution of per-subject Spearman's correlations between predicted and true stroke volume variations. \texttt{VIR} achieves improved robustness, especially on worse-performing subjects, while maintaining tracking performance comparable to baselines.}
    \label{fig:vitaldb_results}
\end{figure}

\begin{figure}[!t]
    \centering
    \includegraphics[width=\columnwidth]{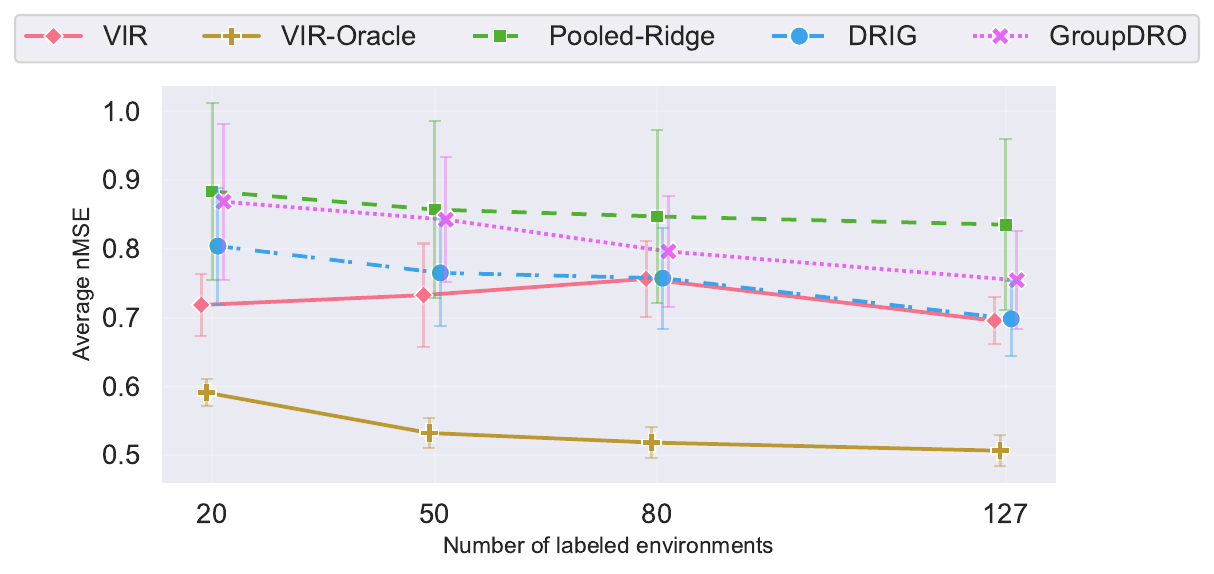} 
\caption{Performance on VitalDB dataset as a function of the number of labeled environments. \texttt{VIR} improves over baselines across all settings. The improvement is most pronounced when few labeled environments are available.}
    \label{fig:vital_lineplots}
\end{figure}

\paragraph{Results.} Figure~\ref{fig:vitaldb_robustness} shows the robustness analysis when all 128 training subjects have labeled data, where each point represents the CVaR at a given quantile threshold. The \texttt{Pooled-Ridge} baseline exhibits poor robustness; its error increases substantially for more difficult subjects. \texttt{GroupDRO} and \texttt{DRIG} improve over \texttt{Pooled-Ridge} but still show relatively poor performance on difficult subjects. \texttt{VIR} achieves lower nMSE across all quantiles, with the improvement most pronounced on difficult subjects. Since difficult subjects correspond to those with larger shifts, this is consistent with VIR's objective of penalizing sensitivity to such shifts.

A natural follow-up question is whether these robustness gains compromise tracking performance. Regularization could improve worst-case nMSE simply by shrinking predictions toward zero rather than by learning robust structure.
Figure~\ref{fig:vitaldb_corr} shows the distribution of per-subject Spearman correlations. \texttt{VIR} achieves median correlation comparable to \texttt{Pooled-Ridge}, with both methods centered around 0.8. The robustness improvements therefore do not come at the cost of tracking accuracy.

Lastly, Figure~\ref{fig:vital_lineplots} shows average nMSE as a function of the number of labeled environments. \texttt{VIR} improves (or matches) over the baselines in all settings. The improvement is most pronounced when the number of labeled environments is small: in this regime, the regularization matrix estimated from unlabeled data provides information on environment perturbations that may not be possible to infer from labeled data alone. When all training subjects have labeled data, \texttt{VIR} still shows improvements over the baselines, which suggests that our approach yields benefits beyond the semi-supervised setting. However, how to best tune the hyperparameter $\gamma$ remains an open problem, as the gap between \texttt{VIR} and \texttt{VIR-Oracle} is still significant.

\subsection{Violations of the causal assumptions}\label{sec:violated_assumptions}
We perform two experiments designed to probe how our methods behave when the assumptions of Setting~\ref{setting:anti-causal} are violated. Full details and figures are in Appendix~\ref{app:violated_assumptions}.

\paragraph{Violation of the anti-causal direction.}
We use the Light Tunnel data (Section~\ref{sec:causal_chamber}) but reverse the prediction direction: we predict a light-intensity measurement (\texttt{ir\_1}) from the RGB brightness settings. This violates the anti-causal assumption because the predictors $X$ (the RGB brightness) now affect the outcome $Y$ (light-intensity). Empirically, MIR collapses to Pooled-Ridge: the leave-one-environment-out cross-validation procedure selects $\gamma \approx 0$, so MIR's predictions coincide with those of Pooled-Ridge. This shows that potential violations of the anti-causal assumption, to certain extent, can be realized from data.

\paragraph{Violation of no effect from $E$ to $Y$.}
We repeat the VitalDB experiment (Section~\ref{sec:vital_db}) without subject-centering. In this setting, the environment $E$ (subject) affects $Y$ (absolute stroke volume) via baseline physiology, violating Setting~\ref{setting:anti-causal}. We apply the combined MIR-VIR objective from Appendix~\ref{app:MIR_VIR} to account for both mean and covariance shifts. The result (Figure~\ref{fig:vital_noCT} in Appendix~\ref{app:violated_assumptions}) shows that MIR-VIR still improves over the baselines, but the gap is substantially smaller than under subject-centering. This is consistent with our expectation: the regularization continues to help by penalizing sensitivity to between-subject variation in $X$, but it cannot eliminate the bias introduced by direct effects of the environment on $Y$.

\section{Conclusion and future work}\label{sec:conclusion}
We presented a framework for anti-causal learning that leverages unlabeled multi-environment data to improve domain generalization. Our proposed methods have worst-case optimality guarantees and admit closed-form estimators. Experiments on a real-world physical system and a biomedical application demonstrate the effectiveness of our approach. This work opens avenues for improving domain generalization in settings where obtaining labels from multiple environments is difficult but unlabeled data are abundant.

Our methods focus on the regression setting (although they can be extended to other losses; see Section~\ref{sec:extension}), with experiments (Section~\ref{sec:vital_db}) on real-world distributional shifts rather than synthetic ones. We believe that domain generalization in regression settings remains comparatively under-explored relative to classification, and hope that the experimental setups we consider in this paper can serve as a useful benchmark for regression-based domain generalization.

We close by outlining several promising directions for future work. (i) Extending MIR/VIR-style regularization to nonparametric methods (e.g., via kernel methods) or to tree-based methods (e.g., gradient-boosted trees) is an interesting open problem. (ii) Under a linear model class with mean squared error loss, only the first two moments of the environment perturbation affect the test loss. For nonlinear models and other losses, higher-moment shifts may also affect the test loss, and extending our framework to account for these (e.g., via a distributional metric) is a natural extension. (iii) Deriving finite-sample generalization bounds that explicitly trade off the number of labeled and unlabeled environments would strengthen the theoretical foundation.

\section*{Acknowledgements}
 We thank Maria Cervera, Antoine Wehenkel, and Achille Nazaret for helpful discussions and feedback.

\section*{Impact Statement}

This paper presents work towards making machine learning models more robust and safe for deployment under distribution shift. By developing methods that leverage unlabeled data to improve worst-case performance in unseen environments, our contributions have potential positive societal impact in applications where reliability is critical.

We do not foresee ethical concerns with this work. The methods are general-purpose tools for robust prediction, the datasets used in all experiments have been peer-reviewed and are publicly available, and the computational requirements are minimal.

\bibliography{refs}

\begin{thebibliography}{41}
\providecommand{\natexlab}[1]{#1}
\providecommand{\url}[1]{\texttt{#1}}
\expandafter\ifx\csname urlstyle\endcsname\relax
  \providecommand{\doi}[1]{doi: #1}\else
  \providecommand{\doi}{doi: \begingroup \urlstyle{rm}\Url}\fi

\bibitem[Arjovsky et~al.(2019)Arjovsky, Bottou, Gulrajani, and Lopez-Paz]{Arjovsky2019}
Arjovsky, M., Bottou, L., Gulrajani, I., and Lopez-Paz, D.
\newblock Invariant risk minimization.
\newblock \emph{arXiv preprint arXiv:1907.02893}, 2019.

\bibitem[Benemerito et~al.(2024)Benemerito, Melis, Wehenkel, and Marzo]{Benemerito_2024}
Benemerito, I., Melis, A., Wehenkel, A., and Marzo, A.
\newblock openbf: an open-source finite volume 1d blood flow solver.
\newblock \emph{Physiological Measurement}, 2024.

\bibitem[Blanchard et~al.(2011)Blanchard, Lee, and Scott]{blanchard2011generalizing}
Blanchard, G., Lee, G., and Scott, C.
\newblock Generalizing from several related classification tasks to a new unlabeled sample.
\newblock In \emph{Advances in neural information processing systems}, volume~24, 2011.

\bibitem[B{\"u}hlmann(2020)]{buhlmann2020invariance}
B{\"u}hlmann, P.
\newblock Invariance, causality and robustness.
\newblock \emph{Statistical Science}, 35\penalty0 (3):\penalty0 404--426, 2020.

\bibitem[B{\"u}hlmann \& Meinshausen(2015)B{\"u}hlmann and Meinshausen]{buhlmann2015magging}
B{\"u}hlmann, P. and Meinshausen, N.
\newblock Magging: maximin aggregation for inhomogeneous large-scale data.
\newblock \emph{Proceedings of the IEEE}, 104\penalty0 (1):\penalty0 126--135, 2015.

\bibitem[DeGrave et~al.(2021)DeGrave, Janizek, and Lee]{degrave2021ai}
DeGrave, A.~J., Janizek, J.~D., and Lee, S.-I.
\newblock Ai for radiographic covid-19 detection selects shortcuts over signal.
\newblock \emph{Nature Machine Intelligence}, 3\penalty0 (7):\penalty0 610--619, 2021.

\bibitem[Eastwood et~al.(2023)Eastwood, Singh, Nicolicioiu, Vlastelica~Pogan{\v{c}}i{\'c}, von K{\"u}gelgen, and Sch{\"o}lkopf]{eastwood2023spuriosity}
Eastwood, C., Singh, S., Nicolicioiu, A.~L., Vlastelica~Pogan{\v{c}}i{\'c}, M., von K{\"u}gelgen, J., and Sch{\"o}lkopf, B.
\newblock Spuriosity didn’t kill the classifier: Using invariant predictions to harness spurious features.
\newblock \emph{Advances in Neural Information Processing Systems}, 36:\penalty0 18291--18324, 2023.

\bibitem[Freni et~al.(2025)Freni, Fries, K{\"u}hne, Reichstein, and Peters]{freni2025maximum}
Freni, F., Fries, A., K{\"u}hne, L., Reichstein, M., and Peters, J.
\newblock Maximum risk minimization with random forests.
\newblock \emph{arXiv preprint arXiv:2512.10445}, 2025.

\bibitem[Gamella et~al.(2025)Gamella, Peters, and B{\"u}hlmann]{gamella2025causal}
Gamella, J.~L., Peters, J., and B{\"u}hlmann, P.
\newblock Causal chambers as a real-world physical testbed for ai methodology.
\newblock \emph{Nature Machine Intelligence}, 7\penalty0 (1):\penalty0 107--118, 2025.

\bibitem[Heinze-Deml \& Meinshausen(2021)Heinze-Deml and Meinshausen]{heinze2021conditional}
Heinze-Deml, C. and Meinshausen, N.
\newblock Conditional variance penalties and domain shift robustness.
\newblock \emph{Machine Learning}, 110\penalty0 (2):\penalty0 303--348, 2021.

\bibitem[Heinze-Deml et~al.(2018)Heinze-Deml, Peters, and Meinshausen]{HeinzeDeml2017}
Heinze-Deml, C., Peters, J., and Meinshausen, N.
\newblock Invariant causal prediction for nonlinear models.
\newblock \emph{Journal of Causal Inference}, 6\penalty0 (2):\penalty0 1--35, 2018.

\bibitem[Hu et~al.(2018)Hu, Niu, Sato, and Sugiyama]{hu2018does}
Hu, W., Niu, G., Sato, I., and Sugiyama, M.
\newblock Does distributionally robust supervised learning give robust classifiers?
\newblock In \emph{International Conference on Machine Learning}, pp.\  2029--2037. PMLR, 2018.

\bibitem[Jiang \& Veitch(2022)Jiang and Veitch]{jiang2022invariant}
Jiang, Y. and Veitch, V.
\newblock Invariant and transportable representations for anti-causal domain shifts.
\newblock \emph{Advances in Neural Information Processing Systems}, 35:\penalty0 20782--20794, 2022.

\bibitem[Kook et~al.(2022)Kook, Sick, and B{\"u}hlmann]{kook2022distributional}
Kook, L., Sick, B., and B{\"u}hlmann, P.
\newblock Distributional anchor regression.
\newblock \emph{Statistics and Computing}, 32\penalty0 (3):\penalty0 39, 2022.

\bibitem[Krueger et~al.(2021)Krueger, Caballero, Jacobsen, Zhang, Binas, Zhang, Le~Priol, and Courville]{krueger2021out}
Krueger, D., Caballero, E., Jacobsen, J.-H., Zhang, A., Binas, J., Zhang, D., Le~Priol, R., and Courville, A.
\newblock Out-of-distribution generalization via risk extrapolation (rex).
\newblock In \emph{International conference on machine learning}, pp.\  5815--5826. PMLR, 2021.

\bibitem[Lee et~al.(2022)Lee, Park, Yoon, Yang, Park, and Jung]{lee2022vitaldb}
Lee, H.-C., Park, Y., Yoon, S.~B., Yang, S.~M., Park, D., and Jung, C.-W.
\newblock Vitaldb, a high-fidelity multi-parameter vital signs database in surgical patients.
\newblock \emph{Scientific Data}, 9\penalty0 (1):\penalty0 279, 2022.

\bibitem[Lueckmann et~al.(2017)Lueckmann, Goncalves, Bassetto, {\"O}cal, Nonnenmacher, and Macke]{lueckmann2017flexible}
Lueckmann, J.-M., Goncalves, P.~J., Bassetto, G., {\"O}cal, K., Nonnenmacher, M., and Macke, J.~H.
\newblock Flexible statistical inference for mechanistic models of neural dynamics.
\newblock \emph{Advances in neural information processing systems}, 30, 2017.

\bibitem[Magliacane et~al.(2018)Magliacane, Van~Ommen, Claassen, Bongers, Versteeg, and Mooij]{magliacane2018domain}
Magliacane, S., Van~Ommen, T., Claassen, T., Bongers, S., Versteeg, P., and Mooij, J.~M.
\newblock Domain adaptation by using causal inference to predict invariant conditional distributions.
\newblock \emph{Advances in neural information processing systems}, 31, 2018.

\bibitem[Makar et~al.(2022)Makar, Packer, Moldovan, Blalock, Halpern, and D’Amour]{pmlr-v151-makar22a}
Makar, M., Packer, B., Moldovan, D., Blalock, D., Halpern, Y., and D’Amour, A.
\newblock Causally motivated shortcut removal using auxiliary labels.
\newblock In \emph{International Conference on Artificial Intelligence and Statistics}, pp.\  739--766. PMLR, 2022.

\bibitem[Manduchi et~al.(2024)Manduchi, Wehenkel, Behrmann, Pegolotti, Miller, Sener, Cuturi, Sapiro, and Jacobsen]{manduchi2024leveraging}
Manduchi, L., Wehenkel, A., Behrmann, J., Pegolotti, L., Miller, A.~C., Sener, O., Cuturi, M., Sapiro, G., and Jacobsen, J.-H.
\newblock Leveraging cardiovascular simulations for in-vivo prediction of cardiac biomarkers.
\newblock \emph{arXiv preprint arXiv:2412.17542}, 2024.

\bibitem[Meinshausen \& B{\"u}hlmann(2015)Meinshausen and B{\"u}hlmann]{MeinshausenMiniMax}
Meinshausen, N. and B{\"u}hlmann, P.
\newblock Maximin effects in inhomogeneous large-scale data.
\newblock \emph{Annals of Statistics}, 43\penalty0 (4):\penalty0 1801--1830, 2015.

\bibitem[Muandet et~al.(2013)Muandet, Balduzzi, and Sch{\"o}lkopf]{muandet2013domain}
Muandet, K., Balduzzi, D., and Sch{\"o}lkopf, B.
\newblock Domain generalization via invariant feature representation.
\newblock In \emph{Proceedings of the 30th International Conference on Machine Learning}, pp.\  10--18. PMLR, 2013.

\bibitem[Palumbo et~al.(2025)Palumbo, Saengkyongam, Cervera, Behrmann, Miller, Sapiro, Heinze-Deml, and Wehenkel]{palumbo2025hybrid}
Palumbo, E., Saengkyongam, S., Cervera, M.~R., Behrmann, J., Miller, A.~C., Sapiro, G., Heinze-Deml, C., and Wehenkel, A.
\newblock Hybrid modeling of photoplethysmography for non-invasive monitoring of cardiovascular parameters.
\newblock \emph{arXiv preprint arXiv:2511.14452}, 2025.

\bibitem[Pearl(2009)]{Pearl2009}
Pearl, J.
\newblock \emph{Causality: Models, Reasoning, and Inference}.
\newblock Cambridge University Press, New York, USA, 2nd edition, 2009.

\bibitem[Peters et~al.(2016)Peters, B{\"u}hlmann, and Meinshausen]{peters2016causal}
Peters, J., B{\"u}hlmann, P., and Meinshausen, N.
\newblock Causal inference by using invariant prediction: identification and confidence intervals.
\newblock \emph{Journal of the Royal Statistical Society. Series B (Statistical Methodology)}, pp.\  947--1012, 2016.

\bibitem[Pfister et~al.(2021)Pfister, Williams, Peters, Aebersold, and B{\"u}hlmann]{pfister2021stabilizing}
Pfister, N., Williams, E.~G., Peters, J., Aebersold, R., and B{\"u}hlmann, P.
\newblock Stabilizing variable selection and regression.
\newblock \emph{The Annals of Applied Statistics}, 15\penalty0 (3):\penalty0 1220--1246, 2021.

\bibitem[Rojas-Carulla et~al.(2018)Rojas-Carulla, Sch{\"o}lkopf, Turner, and Peters]{rojas2018invariant}
Rojas-Carulla, M., Sch{\"o}lkopf, B., Turner, R., and Peters, J.
\newblock Invariant models for causal transfer learning.
\newblock \emph{The Journal of Machine Learning Research}, 19\penalty0 (1):\penalty0 1309--1342, 2018.

\bibitem[Rosenfeld et~al.(2022)Rosenfeld, Ravikumar, and Risteski]{rosenfeld2022domain}
Rosenfeld, E., Ravikumar, P., and Risteski, A.
\newblock Domain-adjusted regression or: Erm may already learn features sufficient for out-of-distribution generalization.
\newblock \emph{arXiv preprint arXiv:2202.06856}, 2022.

\bibitem[Rothenh{\"a}usler et~al.(2021)Rothenh{\"a}usler, Meinshausen, B{\"u}hlmann, and Peters]{rothenhausler2021anchor}
Rothenh{\"a}usler, D., Meinshausen, N., B{\"u}hlmann, P., and Peters, J.
\newblock Anchor regression: Heterogeneous data meet causality.
\newblock \emph{Journal of the Royal Statistical Society: Series B (Statistical Methodology)}, 83\penalty0 (2):\penalty0 215--246, 2021.

\bibitem[Saengkyongam et~al.(2022)Saengkyongam, Henckel, Pfister, and Peters]{saengkyongam2022exploiting}
Saengkyongam, S., Henckel, L., Pfister, N., and Peters, J.
\newblock Exploiting independent instruments: Identification and distribution generalization.
\newblock In \emph{International Conference on Machine Learning}, pp.\  18935--18958. PMLR, 2022.

\bibitem[Saengkyongam et~al.(2023)Saengkyongam, Thams, Peters, and Pfister]{saengkyongam2021invariant}
Saengkyongam, S., Thams, N., Peters, J., and Pfister, N.
\newblock Invariant policy learning: A causal perspective.
\newblock \emph{IEEE transactions on pattern analysis and machine intelligence}, 45\penalty0 (7):\penalty0 8606--8620, 2023.

\bibitem[Saengkyongam et~al.(2024{\natexlab{a}})Saengkyongam, Pfister, Klasnja, Murphy, and Peters]{saengkyongam2024effect}
Saengkyongam, S., Pfister, N., Klasnja, P., Murphy, S., and Peters, J.
\newblock Effect-invariant mechanisms for policy generalization.
\newblock \emph{Journal of Machine Learning Research}, 25\penalty0 (34):\penalty0 1--36, 2024{\natexlab{a}}.

\bibitem[Saengkyongam et~al.(2024{\natexlab{b}})Saengkyongam, Rosenfeld, Ravikumar, Pfister, and Peters]{saengkyongam2024identifying}
Saengkyongam, S., Rosenfeld, E., Ravikumar, P.~K., Pfister, N., and Peters, J.
\newblock Identifying representations for intervention extrapolation.
\newblock In \emph{International Conference on Learning Representations}, 2024{\natexlab{b}}.

\bibitem[Sagawa et~al.(2020)Sagawa, Koh, Hashimoto, and Liang]{sagawa2020distributionally}
Sagawa, S., Koh, P.~W., Hashimoto, T.~B., and Liang, P.
\newblock Distributionally robust neural networks for group shifts: On the importance of regularization for worst-case generalization.
\newblock In \emph{International Conference on Learning Representations}, 2020.

\bibitem[Saugel et~al.(2021)Saugel, Kouz, Scheeren, Greiwe, Hoppe, Romagnoli, and de~Backer]{saugel2021cardiac}
Saugel, B., Kouz, K., Scheeren, T.~W., Greiwe, G., Hoppe, P., Romagnoli, S., and de~Backer, D.
\newblock Cardiac output estimation using pulse wave analysis—physiology, algorithms, and technologies: a narrative review.
\newblock \emph{British journal of anaesthesia}, 126\penalty0 (1):\penalty0 67--76, 2021.

\bibitem[Sch{\"o}lkopf et~al.(2012)Sch{\"o}lkopf, Janzing, Peters, Sgouritsa, Zhang, and Mooij]{scholkopf2012causal}
Sch{\"o}lkopf, B., Janzing, D., Peters, J., Sgouritsa, E., Zhang, K., and Mooij, J.
\newblock On causal and anticausal learning.
\newblock \emph{Proceedings of the 29th International Conference on Machine Learning}, pp.\  459--466, 2012.

\bibitem[Shen et~al.(2026)Shen, B{\"u}hlmann, and Taeb]{shen2025causality}
Shen, X., B{\"u}hlmann, P., and Taeb, A.
\newblock Causality-oriented robustness: exploiting general noise interventions in linear structural causal models.
\newblock \emph{Journal of the American Statistical Association (to appear)}, pp.\  1--20, 2026.

\bibitem[{\v{S}}ola et~al.(2025){\v{S}}ola, B{\"u}hlmann, and Shen]{vsola2025causality}
{\v{S}}ola, M., B{\"u}hlmann, P., and Shen, X.
\newblock Causality-inspired robustness for nonlinear models via representation learning.
\newblock \emph{arXiv preprint arXiv:2505.12868}, 2025.

\bibitem[Subbaswamy \& Saria(2020)Subbaswamy and Saria]{subbaswamy2020development}
Subbaswamy, A. and Saria, S.
\newblock From development to deployment: dataset shift, causality, and shift-stable models in health ai.
\newblock \emph{Biostatistics}, 21\penalty0 (2):\penalty0 345--352, 2020.

\bibitem[Veitch et~al.(2021)Veitch, D'Amour, Yadlowsky, and Eisenstein]{veitch2021counterfactual}
Veitch, V., D'Amour, A., Yadlowsky, S., and Eisenstein, J.
\newblock Counterfactual invariance to spurious correlations in text classification.
\newblock \emph{Advances in neural information processing systems}, 34:\penalty0 16196--16208, 2021.

\bibitem[von K{\"u}gelgen et~al.(2025)von K{\"u}gelgen, Ketterer, Shen, Meinshausen, and Peters]{von2025representation}
von K{\"u}gelgen, J., Ketterer, J., Shen, X., Meinshausen, N., and Peters, J.
\newblock Representation learning for distributional perturbation extrapolation.
\newblock \emph{arXiv preprint arXiv:2504.18522}, 2025.

\end{thebibliography}
\bibliographystyle{icml2026}

\newpage
\appendix
\onecolumn

\section{Proofs}\label{app:proofs}

Throughout this section, we work under Setting~\ref{setting:anti-causal}. We use the notation introduced in Section~\ref{sec:method}. 

For the proofs, we additionally define for any environment $e \in \calE$:
\begin{itemize}
    \item $M^e_\varepsilon \coloneqq \EX[\varepsilon_e \varepsilon_e^\top]$, the second moment matrix of the perturbation,
    \item $\bar{M}_\varepsilon \coloneqq \frac{1}{|\calL|} \sum_{e \in \calL} M^{e}_\varepsilon$, the average second moment across labeled environments.
\end{itemize}

\subsection{Preliminary Lemma}

The following lemma provides a decomposition of the mean squared error that separates environment-invariant and environment-specific components.

\begin{lemma}[MSE Decomposition]\label{lem:mse_decomposition}
Under Setting~\ref{setting:anti-causal}, for any $\beta \in \R^d$ and any environment $e \in \calE$,
\[
\EX^e[(Y - \beta^\top X)^2] = R_0(\beta) + \beta^\top M^e_\varepsilon \beta,
\]
where $R_0(\beta) \coloneqq \EX[(Y - \beta^\top f_0(Y, U, \varepsilon_X))^2]$ does not depend on the environment $e$.
\end{lemma}

\begin{proof}
Under Setting~\ref{setting:anti-causal}, we have $X = f_0(Y, U, \varepsilon_X) + \varepsilon_e$. Define $Z \coloneqq f_0(Y, U, \varepsilon_X)$. Expanding the squared error:
\begin{align*}
\EX^e[(Y - \beta^\top X)^2] 
&= \EX[(Y - \beta^\top Z - \beta^\top \varepsilon_e)^2] \\
&= \underbrace{\EX[(Y - \beta^\top Z)^2]}_{\text{(I)}} - 2\underbrace{\EX[(Y - \beta^\top Z) \cdot \beta^\top \varepsilon_e]}_{\text{(II)}} + \underbrace{\EX[(\beta^\top \varepsilon_e)^2]}_{\text{(III)}}.
\end{align*}

We analyze each term separately.

\textit{Term (I):} This equals $R_0(\beta)$ by definition. It depends only on the joint distribution of $(\varepsilon_Y, \varepsilon_U, \varepsilon_X)$, which is invariant across environments by Setting~\ref{setting:anti-causal}.

\textit{Term (II):} By the independence of $\varepsilon_e$ from $(\varepsilon_Y, \varepsilon_U, \varepsilon_X)$ stated in Setting~\ref{setting:anti-causal}, we can factorize the expectation:
\[
\EX[(Y - \beta^\top Z) \cdot \beta^\top \varepsilon_e] = \EX[Y - \beta^\top Z] \cdot \EX[\beta^\top \varepsilon_e].
\]
By Setting~\ref{setting:anti-causal}, $\EX[Z] = \EX[f_0(Y, U, \varepsilon_X)] = 0$ and $\EX[Y] = \EX[g_0(U, \varepsilon_Y)] = 0$. Therefore, $\EX[Y - \beta^\top Z] = 0$, and Term (II) vanishes.

\textit{Term (III):} This equals $\EX[(\beta^\top \varepsilon_e)^2] = \beta^\top \EX[\varepsilon_e \varepsilon_e^\top] \beta = \beta^\top M^e_\varepsilon \beta$.

Combining these results:
\[
\EX^e[(Y - \beta^\top X)^2] = R_0(\beta) + \beta^\top M^e_\varepsilon \beta. \qedhere
\]
\end{proof}

\subsection{Proof of Theorem~\ref{thm:MIR}}

Recall that $\mu^{\tr}_{\varepsilon} \coloneqq \begin{bmatrix} \EX[\varepsilon_{e_1}] & \cdots & \EX[\varepsilon_{e_p}] \end{bmatrix} \in \R^{d \times p}$ denotes the matrix of perturbation means across all training environments $\calE^{\tr} = \{e_1, \ldots, e_p\}$.

\begin{proof}
The proof proceeds in three steps.

\medskip
\noindent\textit{Step 1: Express the MSE for environments in $\calE^\diamond_\gamma$.}

By Lemma~\ref{lem:mse_decomposition}, for any environment $e \in \calE$,
\[
\EX^e[(Y - \beta^\top X)^2] = R_0(\beta) + \beta^\top M^e_\varepsilon \beta.
\]

\medskip
\noindent\textit{Step 2: Compute the supremum over $\calE^\diamond_\gamma$.}

We claim that
\begin{equation}\label{eq:sup_equivalence_MIR}
\sup_{e \in \calE^\diamond_\gamma} \beta^\top M^e_\varepsilon \beta = \beta^\top \bar{M}_\varepsilon \beta + \sup_{A \in \Delta^\diamond_\gamma} \beta^\top A \beta.
\end{equation}

To establish~\eqref{eq:sup_equivalence_MIR}, we show that $\{M^e_\varepsilon - \bar{M}_\varepsilon : e \in \calE^\diamond_\gamma\} = \Delta^\diamond_\gamma$:
\begin{itemize}
    \item By definition of $\calE^\diamond_\gamma$, for every $e \in \calE^\diamond_\gamma$, we have $M^e_\varepsilon - \bar{M}_\varepsilon \in \Delta^\diamond_\gamma$.
    \item Conversely, for every $A \in \Delta^\diamond_\gamma$, the matrix $\bar{M}_\varepsilon + A$ is positive semi-definite (as the sum of two positive semi-definite matrices). Since $\calE$ contains all environments whose perturbation distribution ranges over all probability distributions on $\R^d$, there exists $e \in \calE$ with $M^e_\varepsilon = \bar{M}_\varepsilon + A$. This $e$ belongs to $\calE^\diamond_\gamma$ by definition, and hence $A \in \{M^e_\varepsilon - \bar{M}_\varepsilon : e \in \calE^\diamond_\gamma\}$.
\end{itemize}
This set equality implies:
\begin{align*}
\sup_{e \in \calE^\diamond_\gamma} \beta^\top M^e_\varepsilon \beta 
&= \sup_{e \in \calE^\diamond_\gamma} \beta^\top (\bar{M}_\varepsilon + (M^e_\varepsilon - \bar{M}_\varepsilon)) \beta \\
&= \beta^\top \bar{M}_\varepsilon \beta + \sup_{e \in \calE^\diamond_\gamma} \beta^\top (M^e_\varepsilon - \bar{M}_\varepsilon) \beta \\
&= \beta^\top \bar{M}_\varepsilon \beta + \sup_{A \in \Delta^\diamond_\gamma} \beta^\top A \beta,
\end{align*}
where the last equality uses $\{M^e_\varepsilon - \bar{M}_\varepsilon : e \in \calE^\diamond_\gamma\} = \Delta^\diamond_\gamma$.

Therefore,
\[
\sup_{e \in \calE^\diamond_\gamma} \EX^e[(Y - \beta^\top X)^2] = R_0(\beta) + \sup_{e \in \calE^\diamond_\gamma} \beta^\top M^e_\varepsilon \beta = R_0(\beta) + \beta^\top \bar{M}_\varepsilon \beta + \sup_{A \in \Delta^\diamond_\gamma} \beta^\top A \beta.
\]

We now evaluate $\sup_{A \in \Delta^\diamond_\gamma} \beta^\top A \beta$.

\textit{Upper bound:} For any $A \in \Delta^\diamond_\gamma$, the constraint $0 \preceq A \preceq \gamma \Var(\mu^{\tr}_{\varepsilon})$ implies
\[
\beta^\top A \beta \leq \gamma \beta^\top \Var(\mu^{\tr}_{\varepsilon}) \beta.
\]

\textit{Attainability:} Define $A^* \coloneqq \gamma \Var(\mu^{\tr}_{\varepsilon})$. We verify that $A^* \in \Delta^\diamond_\gamma$:
\begin{itemize}
    \item $A^*$ is symmetric since $\Var(\mu^{\tr}_{\varepsilon})$ is a covariance matrix.
    \item $A^*$ is positive semi-definite since $\Var(\mu^{\tr}_{\varepsilon})$ is positive semi-definite.
    \item $A^* \preceq \gamma \Var(\mu^{\tr}_{\varepsilon})$ holds with equality.
\end{itemize}
Therefore, $A^* \in \Delta^\diamond_\gamma$, and the upper bound is achieved:
\begin{equation}\label{eq:sup_A_MIR}
\sup_{A \in \Delta^\diamond_\gamma} \beta^\top A \beta = \gamma \beta^\top \Var(\mu^{\tr}_{\varepsilon}) \beta.
\end{equation}

Combining the above,
\begin{equation}\label{eq:worst_case_mse_MIR}
\sup_{e \in \calE^\diamond_\gamma} \EX^e[(Y - \beta^\top X)^2] = R_0(\beta) + \beta^\top \bar{M}_\varepsilon \beta + \gamma \beta^\top \Var(\mu^{\tr}_{\varepsilon}) \beta.
\end{equation}

\medskip
\noindent\textit{Step 3: Relate to the MIR objective.}

We compute the average training MSE over labeled environments. By Lemma~\ref{lem:mse_decomposition},
\begin{align}
\EX^{\tr}[(Y - \beta^\top X)^2] &= \frac{1}{|\calL|} \sum_{e \in \calL} \EX^{e}[(Y - \beta^\top X)^2] \nonumber \\
&= \frac{1}{|\calL|} \sum_{e \in \calL} \left( R_0(\beta) + \beta^\top M^{e}_\varepsilon \beta \right) \nonumber \\
&= R_0(\beta) + \beta^\top \bar{M}_\varepsilon \beta. \label{eq:training_mse_MIR}
\end{align}

Substituting~\eqref{eq:training_mse_MIR} into~\eqref{eq:worst_case_mse_MIR},
\begin{equation}\label{eq:worst_case_equals_MIR}
\sup_{e \in \calE^\diamond_\gamma} \EX^e[(Y - \beta^\top X)^2] = \EX^{\tr}[(Y - \beta^\top X)^2] + \gamma \beta^\top \Var(\mu^{\tr}_{\varepsilon}) \beta.
\end{equation}

Finally, we show that $K = \mu^{\tr}_{\varepsilon}$. Under Setting~\ref{setting:anti-causal}, for each training environment $e_i$,
\[
\EX^{e_i}[X] = \EX[f_0(Y, U, \varepsilon_X) + \varepsilon_{e_i}] = \EX[f_0(Y, U, \varepsilon_X)] + \EX[\varepsilon_{e_i}].
\]
Since $\EX[f_0(Y, U, \varepsilon_X)] = 0$ by Setting~\ref{setting:anti-causal}, we have $\EX^{e_i}[X] = \EX[\varepsilon_{e_i}]$. Therefore,
\[
K = \begin{bmatrix} \EX^{e_1}[X] & \cdots & \EX^{e_p}[X] \end{bmatrix} = \begin{bmatrix} \EX[\varepsilon_{e_1}] & \cdots & \EX[\varepsilon_{e_p}] \end{bmatrix} = \mu^{\tr}_{\varepsilon}.
\]

Since $K = \mu^{\tr}_{\varepsilon}$, we have $\Var(K) = \Var(\mu^{\tr}_{\varepsilon})$. Substituting into~\eqref{eq:worst_case_equals_MIR},
\[
\sup_{e \in \calE^\diamond_\gamma} \EX^e[(Y - \beta^\top X)^2] = \EX^{\tr}[(Y - \beta^\top X)^2] + \gamma \, \beta^\top \Var(K) \beta.
\]

The right-hand side is precisely the MIR objective~\eqref{eq:MIR_objective}. Since the equality holds for all $\beta \in \R^d$, minimizing over $\beta$ on both sides yields
\[
\beta^{\textup{MIR}}_\gamma = \argmin_{\beta \in \R^d} \left\{ \EX^{\tr}[(Y - \beta^\top X)^2] + \gamma \, \beta^\top \Var(K) \beta \right\} = \argmin_{\beta \in \R^d} \sup_{e \in \calE^\diamond_\gamma} \EX^e[(Y - \beta^\top X)^2].
\]
\end{proof}

\subsection{Proof of Theorem~\ref{thm:VIR}}

Recall the notation from Section~\ref{sec:method}: for each environment $e_i \in \calE^{\tr}$, $G^{e_i}_X$ denotes the covariance matrix of $X$, $G^{e_i}_\varepsilon$ denotes the covariance matrix of $\varepsilon_{e_i}$, and $\bar{G}_X \coloneqq \frac{1}{p}\sum_{i=1}^p G^{e_i}_X$ and $\bar{G}_\varepsilon \coloneqq \frac{1}{p}\sum_{i=1}^p G^{e_i}_\varepsilon$ denote their averages across all training environments.

\begin{lemma}\label{lem:V_psd}
The matrix $V_\varepsilon \coloneqq \frac{1}{p} \sum_{i=1}^{p} (G^{e_i}_X - \bar{G}_X)^2$ is symmetric and positive semi-definite.
\end{lemma}

\begin{proof}
Each $G^{e_i}_X - \bar{G}_X$ is symmetric (as a difference of covariance matrices). The square of a symmetric matrix is symmetric and positive semi-definite. The sum of positive semi-definite matrices is positive semi-definite.
\end{proof}

\begin{proof}[Proof of Theorem~\ref{thm:VIR}]
The proof follows the same structure as the proof of Theorem~\ref{thm:MIR}.

\medskip
\noindent\textit{Step 1: Express the MSE for environments in $\calE^\dagger_\gamma$.}

By Lemma~\ref{lem:mse_decomposition}, for any environment $e \in \calE$,
\[
\EX^e[(Y - \beta^\top X)^2] = R_0(\beta) + \beta^\top M^e_\varepsilon \beta.
\]

\medskip
\noindent\textit{Step 2: Compute the supremum over $\calE^\dagger_\gamma$.}

We claim that
\begin{equation}\label{eq:sup_equivalence_VIR}
\sup_{e \in \calE^\dagger_\gamma} \beta^\top M^e_\varepsilon \beta = \beta^\top \bar{M}_\varepsilon \beta + \sup_{A \in \Delta^\dagger_\gamma} \beta^\top A \beta.
\end{equation}

To establish~\eqref{eq:sup_equivalence_VIR}, we show that $\{M^e_\varepsilon - \bar{M}_\varepsilon : e \in \calE^\dagger_\gamma\} = \Delta^\dagger_\gamma$:
\begin{itemize}
    \item By definition of $\calE^\dagger_\gamma$, for every $e \in \calE^\dagger_\gamma$, we have $M^e_\varepsilon - \bar{M}_\varepsilon \in \Delta^\dagger_\gamma$.
    \item Conversely, for every $A \in \Delta^\dagger_\gamma$, the matrix $\bar{M}_\varepsilon + A$ is positive semi-definite (as the sum of two positive semi-definite matrices). Since $\calE$ contains all environments whose perturbation distribution ranges over all probability distributions on $\R^d$, there exists $e \in \calE$ with $M^e_\varepsilon = \bar{M}_\varepsilon + A$. This $e$ belongs to $\calE^\dagger_\gamma$ by definition, and hence $A \in \{M^e_\varepsilon - \bar{M}_\varepsilon : e \in \calE^\dagger_\gamma\}$.
\end{itemize}
This set equality implies:
\begin{align*}
\sup_{e \in \calE^\dagger_\gamma} \beta^\top M^e_\varepsilon \beta 
&= \sup_{e \in \calE^\dagger_\gamma} \beta^\top (\bar{M}_\varepsilon + (M^e_\varepsilon - \bar{M}_\varepsilon)) \beta \\
&= \beta^\top \bar{M}_\varepsilon \beta + \sup_{e \in \calE^\dagger_\gamma} \beta^\top (M^e_\varepsilon - \bar{M}_\varepsilon) \beta \\
&= \beta^\top \bar{M}_\varepsilon \beta + \sup_{A \in \Delta^\dagger_\gamma} \beta^\top A \beta,
\end{align*}
where the last equality uses $\{M^e_\varepsilon - \bar{M}_\varepsilon : e \in \calE^\dagger_\gamma\} = \Delta^\dagger_\gamma$.

Therefore,
\[
\sup_{e \in \calE^\dagger_\gamma} \EX^e[(Y - \beta^\top X)^2] = R_0(\beta) + \sup_{e \in \calE^\dagger_\gamma} \beta^\top M^e_\varepsilon \beta = R_0(\beta) + \beta^\top \bar{M}_\varepsilon \beta + \sup_{A \in \Delta^\dagger_\gamma} \beta^\top A \beta.
\]

We now evaluate $\sup_{A \in \Delta^\dagger_\gamma} \beta^\top A \beta$.

\textit{Upper bound:} For any $A \in \Delta^\dagger_\gamma$, the constraint $0 \preceq A \preceq \gamma V_\varepsilon$ implies
\[
\beta^\top A \beta \leq \gamma \beta^\top V_\varepsilon \beta.
\]

\textit{Attainability:} Define $A^* \coloneqq \gamma V_\varepsilon$. By Lemma~\ref{lem:V_psd}, $V_\varepsilon$ is symmetric and positive semi-definite, so $A^*$ is as well. The constraint $A^* \preceq \gamma V_\varepsilon$ holds with equality. Therefore, $A^* \in \Delta^\dagger_\gamma$, and the upper bound is achieved:
\begin{equation}\label{eq:sup_A_VIR}
\sup_{A \in \Delta^\dagger_\gamma} \beta^\top A \beta = \gamma \beta^\top V_\varepsilon \beta.
\end{equation}

Combining the above,
\begin{equation}\label{eq:worst_case_mse_VIR}
\sup_{e \in \calE^\dagger_\gamma} \EX^e[(Y - \beta^\top X)^2] = R_0(\beta) + \beta^\top \bar{M}_\varepsilon \beta + \gamma \beta^\top V_\varepsilon \beta.
\end{equation}

\medskip
\noindent\textit{Step 3: Relate to the VIR objective.}

By the same calculation as in~\eqref{eq:training_mse_MIR},
\begin{equation}\label{eq:training_mse_VIR}
\EX^{\tr}[(Y - \beta^\top X)^2] = \frac{1}{|\calL|} \sum_{e \in \calL} \EX^{e}[(Y - \beta^\top X)^2] = R_0(\beta) + \beta^\top \bar{M}_\varepsilon \beta.
\end{equation}

Substituting~\eqref{eq:training_mse_VIR} into~\eqref{eq:worst_case_mse_VIR},
\begin{equation}\label{eq:worst_case_equals_VIR_intermediate}
\sup_{e \in \calE^\dagger_\gamma} \EX^e[(Y - \beta^\top X)^2] = \EX^{\tr}[(Y - \beta^\top X)^2] + \gamma \beta^\top V_\varepsilon \beta.
\end{equation}

It remains to show that $V_\varepsilon = \frac{1}{p} \sum_{i=1}^{p} (G^{e_i}_X - \bar{G}_X)^2$.

Under Setting~\ref{setting:anti-causal}, we have $X = Z + \varepsilon_e$ where $Z \coloneqq f_0(Y, U, \varepsilon_X)$. Since $\varepsilon_e$ is independent of $(\varepsilon_Y, \varepsilon_U, \varepsilon_X)$, the covariance of $X$ in environment $e_i$ decomposes as:
\begin{equation}\label{eq:cov_X_decomp}
G^{e_i}_X = G_Z + G^{e_i}_\varepsilon,
\end{equation}
where $G_Z \coloneqq \Var(Z)$ denotes the covariance of $Z$, which does not depend on the environment.

Taking the average across all training environments,
\[
\bar{G}_X = \frac{1}{p} \sum_{i=1}^{p} G^{e_i}_X = G_Z + \bar{G}_\varepsilon.
\]

Therefore, for each training environment $e_i$,
\begin{equation}\label{eq:G_X_centered}
G^{e_i}_X - \bar{G}_X = G^{e_i}_\varepsilon - \bar{G}_\varepsilon.
\end{equation}

Consequently,
\begin{equation}\label{eq:V_equality}
\frac{1}{p} \sum_{i=1}^{p} (G^{e_i}_X - \bar{G}_X)^2 = \frac{1}{p} \sum_{i=1}^{p} (G^{e_i}_\varepsilon - \bar{G}_\varepsilon)^2 = V_\varepsilon.
\end{equation}

Substituting into~\eqref{eq:worst_case_equals_VIR_intermediate},
\[
\sup_{e \in \calE^\dagger_\gamma} \EX^e[(Y - \beta^\top X)^2] = \EX^{\tr}[(Y - \beta^\top X)^2] + \gamma \beta^\top \left( \frac{1}{p} \sum_{i=1}^{p} (G^{e_i}_X - \bar{G}_X)^2 \right) \beta.
\]

The right-hand side is precisely the VIR objective~\eqref{eq:VIR_objective}. Since the equality holds for all $\beta \in \R^d$, minimizing over $\beta$ on both sides yields
\[
\beta^{\textup{VIR}}_\gamma = \argmin_{\beta \in \R^d} \left\{ \EX^{\tr}[(Y - \beta^\top X)^2] + \gamma \beta^\top \left( \frac{1}{p} \sum_{i=1}^{p} (G^{e_i}_X - \bar{G}_X)^2 \right) \beta \right\} = \argmin_{\beta \in \R^d} \sup_{e \in \calE^\dagger_\gamma} \EX^e[(Y - \beta^\top X)^2].
\]
\end{proof}

\subsection{Proof of Proposition~\ref{prop:consistency}}\label{proof:prop:consistency}

\begin{proof}
Recall the assumptions of Proposition~\ref{prop:consistency}:
\begin{enumerate}[label=(\roman*)]
    \item $\EX^{e}[\|X\|_2^2] < \infty$ for all $e \in \calE^{\tr}$,
    \item $\EX^{e}[Y^2] < \infty$ for all $e \in \calL$,
    \item $\EX^{\tr}[XX^\top] + \gamma H$ is non-singular.
\end{enumerate}

For each labeled environment $e \in \calL$, let $k_e$ denote the number of samples from environment $e$, so that $k = \sum_{e \in \calL} k_e$. We assume balanced sampling, i.e., $k_e / k \to 1/|\calL|$ for each $e \in \calL$ (the proof for unbalanced sampling follows analogously with $\EX^{\tr}[\cdot]$ defined as the corresponding weighted average). The finite-sample estimator admits the closed-form solution:
\[
\hat{\beta}_{\gamma} = \left( \frac{1}{k}\mathbf{X}^\top \mathbf{X} + \gamma \hat{H} \right)^{-1} \left( \frac{1}{k}\mathbf{X}^\top \mathbf{Y} \right).
\]

We establish the convergence in probability of each component.

\begin{enumerate}
    \item \textbf{Labeled data convergence:} The sample second moment from the labeled data can be written as:
    \[
    \frac{1}{k} \mathbf{X}^\top \mathbf{X} = \sum_{e \in \calL} \frac{k_e}{k} \left( \frac{1}{k_e} \sum_{i: E_i = e} X_i X_i^\top \right).
    \]
    By assumption (i), $\EX^{e}[\|X\|_2^2] < \infty$ for all $e \in \calL \subseteq \calE^{\tr}$, which implies $\EX^{e}[XX^\top]$ has finite entries. The Weak Law of Large Numbers then implies that for each $e \in \calL$,
    \[
    \frac{1}{k_e} \sum_{i: E_i = e} X_i X_i^\top \xrightarrow{p} \EX^{e}[X X^\top] \quad \text{as } k_e \to \infty.
    \]
    Since $k_e / k \to 1/|\calL|$ by the balanced sampling assumption, it follows that:
    \[
    \frac{1}{k} \mathbf{X}^\top \mathbf{X} \xrightarrow{p} \frac{1}{|\calL|} \sum_{e \in \calL} \EX^{e}[X X^\top] = \EX^{\tr}[X X^\top].
    \]
    
    Similarly, by assumptions (i) and (ii), the Weak Law of Large Numbers then yields:
    \[
    \frac{1}{k} \mathbf{X}^\top \mathbf{Y} \xrightarrow{p} \EX^{\tr}[XY].
    \]

    \item \textbf{Unlabeled data convergence:} By assumption (i), $\EX^{e}[\|X\|_2^2] < \infty$ for all $e \in \calE^{\tr}$, which ensures that the first and second moments of $X$ are finite in each training environment. The Weak Law of Large Numbers implies that for each environment $e_i \in \calE^{\tr}$, as $n_i \to \infty$:
    \begin{itemize}
        \item For MIR: $\hat{\mu}^{e_i}_X \xrightarrow{p} \EX^{e_i}[X]$, and hence $\hat{K} \xrightarrow{p} K$.
        \item For VIR: $\hat{G}^{e_i}_X \xrightarrow{p} G^{e_i}_X$, and hence $\hat{\bar{G}}_X \xrightarrow{p} \bar{G}_X$.
    \end{itemize}
    The matrix $\hat{H}$ is a continuous function of these environment-specific moments: for MIR, $\hat{H} = \Var(\hat{K})$, and for VIR, $\hat{H} = \frac{1}{p}\sum_{i=1}^{p}(\hat{G}^{e_i}_X - \hat{\bar{G}}_X)^2$. By the Continuous Mapping Theorem, $\hat{H} \xrightarrow{p} H$.

    \item \textbf{Combined convergence:} By the properties of convergence in probability (specifically, that convergence in probability is preserved under addition),
    \[
    \frac{1}{k}\mathbf{X}^\top \mathbf{X} + \gamma \hat{H} \xrightarrow{p} \EX^{\tr}[X X^\top] + \gamma H.
    \]
    By assumption (iii), $\EX^{\tr}[X X^\top] + \gamma H$ is non-singular. Since matrix inversion is continuous at non-singular matrices, a final application of the Continuous Mapping Theorem yields:
    \[
    \hat{\beta}_{\gamma} \xrightarrow{p} (\EX^{\tr}[X X^\top] + \gamma H)^{-1} \EX^{\tr}[XY] = \beta_{\gamma}.
    \]
\end{enumerate}
This concludes the proof.
\end{proof}

\section{Derivation of VIR under shared eigenbasis}\label{app:VIR_shared_eigenbasis}

We derive the simplification of the VIR regularization term under the assumption that the perturbation covariance matrices $\{G^{e_i}_{\varepsilon}\}_{i=1}^{p}$ share a common eigenbasis.

Suppose there exists an orthogonal matrix $Q \in \R^{d \times d}$ such that for all $i \in \{1, \ldots, p\}$,
\[
G^{e_i}_{\varepsilon} = Q \Lambda^{e_i} Q^\top,
\]
where $\Lambda^{e_i} = \diag(\lambda^{e_i}_1, \ldots, \lambda^{e_i}_d)$ with $\lambda^{e_i}_j \geq 0$ for all $j$.

Under Setting~\ref{setting:anti-causal}, we have $X = Z + \varepsilon_e$ where $Z := f_0(Y, U, \varepsilon_X)$. Since $\varepsilon_e$ is independent of $(Y, U, \varepsilon_X)$, the covariance of $X$ in environment $e_i$ decomposes as
\[
G^{e_i}_X = G_Z + G^{e_i}_{\varepsilon},
\]
where $G_Z := \Var(Z)$ does not depend on the environment (see the proof of Theorem~\ref{thm:VIR}). Therefore,
\[
G^{e_i}_X - \bar{G}_X = G^{e_i}_{\varepsilon} - \bar{G}_{\varepsilon}.
\]

The average perturbation covariance matrix is
\[
\bar{G}_{\varepsilon} = \frac{1}{p} \sum_{i=1}^{p} G^{e_i}_{\varepsilon} = \frac{1}{p} \sum_{i=1}^{p} Q \Lambda^{e_i} Q^\top = Q \left( \frac{1}{p} \sum_{i=1}^{p} \Lambda^{e_i} \right) Q^\top = Q \bar{\Lambda} Q^\top,
\]
where $\bar{\Lambda} = \diag(\bar{\lambda}_1, \ldots, \bar{\lambda}_d)$ and $\bar{\lambda}_j = \frac{1}{p} \sum_{i=1}^{p} \lambda^{e_i}_j$.

The deviation from the average is
\[
G^{e_i}_{\varepsilon} - \bar{G}_{\varepsilon} = Q \Lambda^{e_i} Q^\top - Q \bar{\Lambda} Q^\top = Q (\Lambda^{e_i} - \bar{\Lambda}) Q^\top.
\]

Since $\Lambda^{e_i} - \bar{\Lambda} = \diag(\lambda^{e_i}_1 - \bar{\lambda}_1, \ldots, \lambda^{e_i}_d - \bar{\lambda}_d)$ is diagonal, its square is
\[
(\Lambda^{e_i} - \bar{\Lambda})^2 = \diag\left( (\lambda^{e_i}_1 - \bar{\lambda}_1)^2, \ldots, (\lambda^{e_i}_d - \bar{\lambda}_d)^2 \right).
\]

Therefore,
\[
(G^{e_i}_{\varepsilon} - \bar{G}_{\varepsilon})^2 = Q (\Lambda^{e_i} - \bar{\Lambda}) Q^\top Q (\Lambda^{e_i} - \bar{\Lambda}) Q^\top = Q (\Lambda^{e_i} - \bar{\Lambda})^2 Q^\top,
\]
where we used the orthogonality of $Q$, i.e., $Q^\top Q = I$.

Averaging over environments,
\begin{align*}
\frac{1}{p} \sum_{i=1}^{p} (G^{e_i}_{\varepsilon} - \bar{G}_{\varepsilon})^2 &= \frac{1}{p} \sum_{i=1}^{p} Q (\Lambda^{e_i} - \bar{\Lambda})^2 Q^\top \\
&= Q \left( \frac{1}{p} \sum_{i=1}^{p} (\Lambda^{e_i} - \bar{\Lambda})^2 \right) Q^\top \\
&= Q \diag(\sigma^2_1, \ldots, \sigma^2_d) Q^\top,
\end{align*}
where
\[
\sigma^2_j = \frac{1}{p} \sum_{i=1}^{p} (\lambda^{e_i}_j - \bar{\lambda}_j)^2
\]
is the variance of the $j$-th eigenvalue across environments.

Since $G^{e_i}_X - \bar{G}_X = G^{e_i}_{\varepsilon} - \bar{G}_{\varepsilon}$, we have
\[
\frac{1}{p} \sum_{i=1}^{p} (G^{e_i}_X - \bar{G}_X)^2 = \frac{1}{p} \sum_{i=1}^{p} (G^{e_i}_{\varepsilon} - \bar{G}_{\varepsilon})^2 = Q \diag(\sigma^2_1, \ldots, \sigma^2_d) Q^\top.
\]

The VIR regularization term is therefore
\begin{align*}
\beta^\top \left( \frac{1}{p} \sum_{i=1}^{p} (G^{e_i}_X - \bar{G}_X)^2 \right) \beta &= \beta^\top Q \diag(\sigma^2_1, \ldots, \sigma^2_d) Q^\top \beta \\
&= \tilde{\beta}^\top \diag(\sigma^2_1, \ldots, \sigma^2_d) \tilde{\beta} \\
&= \sum_{j=1}^{d} \sigma^2_j (\tilde{\beta}_j)^2,
\end{align*}
where $\tilde{\beta} = Q^\top \beta$ denotes the coefficients in the shared eigenbasis.

This shows that VIR applies regularization proportional to $\sigma^2_j$ along the $j$-th eigenvector direction. Directions with larger variance in eigenvalues across environments receive stronger regularization, encouraging the model to rely on directions where the covariance structure is more stable.

\section{Comparison of VIR formulations}\label{app:VIR_alternative}

We compare two formulations of the VIR regularization. Let $A_i \coloneqq G^{e_i}_X - \bar{G}_X$ for $i \in \{1, \ldots, p\}$ denote the deviation of the covariance matrix in environment $e_i$ from the average. Note that each $A_i$ is symmetric.

Current VIR as in Section~\ref{sec:VIR}:
\[
R_{\text{current}}(\beta) \coloneqq \beta^\top \left( \frac{1}{p} \sum_{i=1}^{p} A_i^2 \right) \beta = \frac{1}{p} \sum_{i=1}^{p} \beta^\top A_i^2 \beta = \frac{1}{p} \sum_{i=1}^{p} \|A_i \beta\|_2^2,
\]
where the last equality follows from the symmetry of $A_i$.

Alternative VIR:
\[
R_{\text{alt}}(\beta) \coloneqq \frac{1}{p} \sum_{i=1}^{p} \left( \Var^{e_i}(\beta^\top X) - \frac{1}{p} \sum_{i=1}^p \Var^{e_i}(\beta^\top X) \right)^2 = \frac{1}{p} \sum_{i=1}^{p} \left( \beta^\top A_i \beta \right)^2.
\]

The key difference is that $R_{\text{current}}$ penalizes the norm of $A_i \beta$, while $R_{\text{alt}}$ penalizes the square of the quadratic form $\beta^\top A_i \beta$. Since $A_i$ is typically indefinite (having both positive and negative eigenvalues), there may exist directions $\beta$ for which the positive and negative contributions cancel, yielding $\beta^\top A_i \beta = 0$ even though $A_i \beta \neq 0$. Such directions receive no regularization under $R_{\text{alt}}$ but are regularized under $R_{\text{current}}$.

\paragraph{Example: The risk of alternative VIR.}

Consider $d = 2$ covariates and $p = 2$ environments with the following covariance matrices:
\[
G^{e_1}_X = \begin{pmatrix} 2 & 0 \\ 0 & 1 \end{pmatrix}, \quad G^{e_2}_X = \begin{pmatrix} 1 & 0 \\ 0 & 2 \end{pmatrix}.
\]
The average is:
\[
\bar{G}_X = \frac{1}{2}(G^{e_1}_X + G^{e_2}_X) = \begin{pmatrix} 1.5 & 0 \\ 0 & 1.5 \end{pmatrix},
\]
and the deviations are:
\[
A_1 = G^{e_1}_X - \bar{G}_X = \begin{pmatrix} 0.5 & 0 \\ 0 & -0.5 \end{pmatrix}, \quad A_2 = G^{e_2}_X - \bar{G}_X = \begin{pmatrix} -0.5 & 0 \\ 0 & 0.5 \end{pmatrix}.
\]

\textbf{Alternative VIR.} Consider $\beta = (1, 1)^\top$. We compute:
\[
\beta^\top A_1 \beta = \begin{pmatrix} 1 & 1 \end{pmatrix} \begin{pmatrix} 0.5 & 0 \\ 0 & -0.5 \end{pmatrix} \begin{pmatrix} 1 \\ 1 \end{pmatrix} = \begin{pmatrix} 1 & 1 \end{pmatrix} \begin{pmatrix} 0.5 \\ -0.5 \end{pmatrix} = 0.5 - 0.5 = 0.
\]
Similarly, $\beta^\top A_2 \beta = -0.5 + 0.5 = 0$. Thus:
\[
R_{\text{alt}}(\beta) = \frac{1}{2}\left(0^2 + 0^2\right) = 0,
\]
and $\beta = (1, 1)^\top$ receives no regularization under the alternative VIR.

\textbf{Current VIR.} For the same $\beta = (1, 1)^\top$, we compute:
\[
A_1 \beta = \begin{pmatrix} 0.5 & 0 \\ 0 & -0.5 \end{pmatrix} \begin{pmatrix} 1 \\ 1 \end{pmatrix} = \begin{pmatrix} 0.5 \\ -0.5 \end{pmatrix}, \quad \|A_1 \beta\|_2^2 = 0.5^2 + (-0.5)^2 = 0.5.
\]
Similarly, $A_2 \beta = (-0.5, 0.5)^\top$ and $\|A_2 \beta\|_2^2 = 0.5$. Thus:
\[
R_{\text{current}}(\beta) = \frac{1}{2}(0.5 + 0.5) = 0.5 > 0,
\]
and $\beta = (1, 1)^\top$ is regularized under the current VIR.

\textbf{Test environment.} Now consider a test environment with:
\[
G^{e_{\text{test}}}_X = \begin{pmatrix} 3 & 0 \\ 0 & 1 \end{pmatrix}, \quad A_{\text{test}} = G^{e_{\text{test}}}_X - \bar{G}_X = \begin{pmatrix} 1.5 & 0 \\ 0 & -0.5 \end{pmatrix}.
\]
For $\beta = (1, 1)^\top$:
\[
\beta^\top A_{\text{test}} \beta = \begin{pmatrix} 1 & 1 \end{pmatrix} \begin{pmatrix} 1.5 \\ -0.5 \end{pmatrix} = 1.5 - 0.5 = 1 \neq 0.
\]
The predictor $\beta = (1, 1)^\top$ that was unpenalized by the alternative VIR is sensitive to the test environment shift. In contrast, the current VIR regularizes this $\beta$, encouraging robustness to such shifts.

\paragraph{Summary.}

The current VIR imposes a strictly stronger invariance condition than the alternative VIR in the sense that $R_{\text{current}}(\beta) = 0$ implies $R_{\text{alt}}(\beta) = 0$, but not vice versa. The alternative VIR is vulnerable to cancellation effects when $A_i$ is indefinite, allowing directions $\beta$ with $\beta^\top A_i \beta = 0$ but $A_i \beta \neq 0$ to escape regularization. The current VIR avoids this issue by penalizing $\|A_i \beta\|_2^2$, which is non-negative for each environment. This ensures that no cancellation can occur across environments: the only way to achieve $R_{\text{current}}(\beta) = 0$ is to have $A_i \beta = 0$ for all $i$.

\section{Combined MIR-VIR}\label{app:MIR_VIR}

In practice, environment perturbations may affect both the means and covariances. We can combine both regularization terms in a single objective:
\begin{equation}\label{eq:MIR_VIR_objective}
\begin{aligned}
\beta^{\text{MIR-VIR}}_{\gamma_1, \gamma_2} = \argmin_{\beta \in \R^d} \; &\EX^{\tr}[(Y - \beta^\top X)^2] + \gamma_1 \beta^\top \Var(K)\beta \\
&+ \gamma_2 \beta^\top \left(\frac{1}{p}\sum_{i=1}^p (G^{e_i}_X - \bar{G}_X)^2\right) \beta.
\end{aligned}
\end{equation}
The separate regularization parameters $\gamma_1$ and $\gamma_2$ allow for different degrees of robustness to mean and covariance shifts, respectively.

\section{Extension: Setting~\ref{setting:anti-causal} with observed covariates}\label{app:mediator_extension}
We discuss a potential relaxation of Setting~\ref{setting:anti-causal} in which the environment $E$ may influence $Y$ (or $U$) through observed covariates $W$.

\begin{setting}\label{setting:extend_scm}
We consider the SCM \\
\begin{minipage}{0.49\columnwidth}
\begin{equation*}
\calS(e):
\begin{cases}
W \coloneqq h_0(\varepsilon_e) \\
U \coloneqq u_0(W, \varepsilon_U) \\
Y \coloneqq g_0(W, U, \varepsilon_Y) \\
X \coloneqq f_0(Y, U, W, \varepsilon_X) + \ell_0(\varepsilon_e),
\end{cases}
\end{equation*}
\end{minipage}%
\quad
\begin{minipage}{0.6\columnwidth}
 \begin{tikzpicture}[scale=1.3, node distance=1.0cm, roundnode/.style={circle, draw, inner sep=3pt,minimum size=7mm}, squarenode/.style={rectangle, draw, inner sep=1pt, minimum size=6mm}] \node[squarenode] (E) at (-2, 0.8){$e$}; \node[roundnode] (X) at (-1, 0){$X$}; \node[roundnode] (W) at (-1, 1.5){$W$}; \node[roundnode] (Y) at (1, 0){$Y$}; \node[roundnode][fill=black!25] (U) at (0, 0.8){$U$}; \tikzstyle{EdgeStyle}=[line width=1, -Latex] \Edge[label=$f_0$](Y)(X) \Edge[labelstyle={pos=0.42}](E)(X) \tikzstyle{EdgeStyle}=[bend left=20, line width=1, -Latex, dashed] \Edge[](U)(Y) \tikzstyle{EdgeStyle}=[bend right=20, line width=1, -Latex, dashed] \Edge[](U)(X) \tikzstyle{EdgeStyle}=[line width=1, -Latex] \Edge[](E)(W) \tikzstyle{EdgeStyle}=[line width=1, -Latex] \Edge[](W)(X) \tikzstyle{EdgeStyle}=[bend left=40, line width=1, -Latex] \Edge[](W)(Y) \end{tikzpicture}
\end{minipage}\vspace{0.2cm}\\
where $W \in \R^m$ are observed covariates and $\varepsilon_e$ is the environment-specific perturbation. 
\end{setting}

Conditioning on $W = w$, $Y$ becomes independent of $\varepsilon_e$, and Setting~\ref{setting:extend_scm} reduces to Setting~\ref{setting:anti-causal} with $\varepsilon_e$ replaced by $\ell_0(\varepsilon_e)$. We can therefore extend MIR and VIR to Setting~\ref{setting:extend_scm} by conditioning on $W$: for example, we may discretize $W$ into strata, apply MIR and VIR within each stratum.

In healthcare, $W$ could be demographic variables (e.g., age, sex) or site-specific protocols that mediate the effect of the environment on patient outcomes. We leave a full theoretical and empirical study of this extension to future work.

\section{Stress-tests on violations of causal assumptions}\label{app:violated_assumptions}

\begin{figure}[!t]
    \centering
    \begin{minipage}{.7\linewidth}
    \centering
    \begin{subfigure}[b]{1\columnwidth}
        \includegraphics[width=\columnwidth]{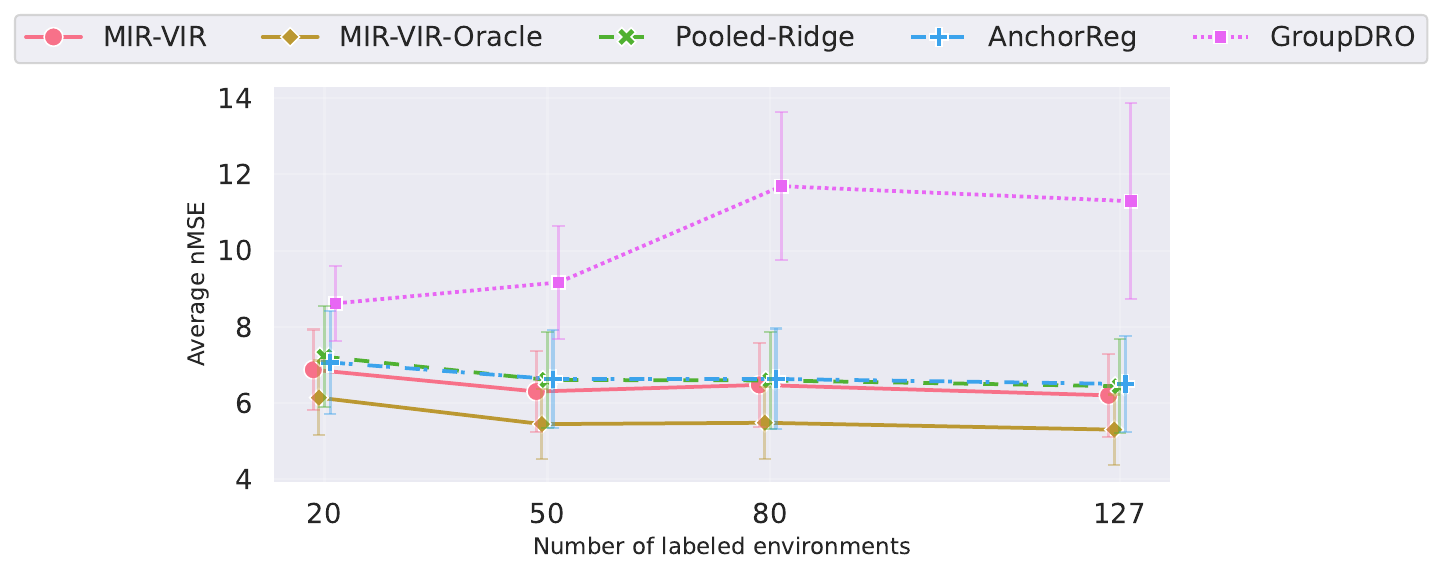}
    \end{subfigure}
    \begin{subfigure}[b]{0.725\columnwidth}
        \includegraphics[width=\columnwidth]{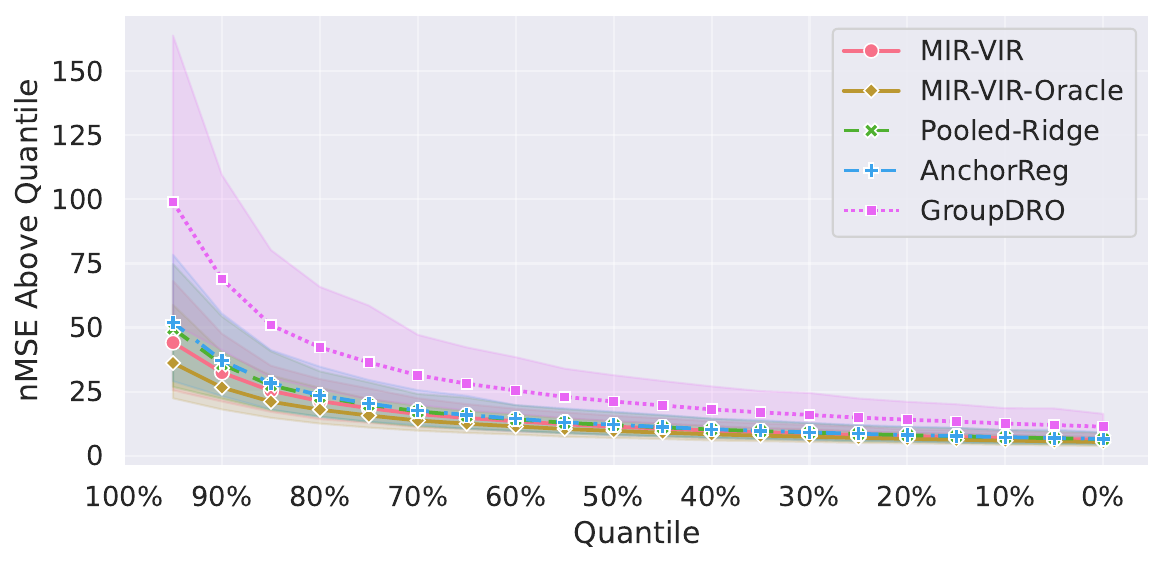}
        \subcaption{VitalDB results without subject-centering}
        \label{fig:vital_noCT}
    \end{subfigure}
    \end{minipage}%
    \begin{minipage}{.5\linewidth}
    \begin{subfigure}[b]{0.6\columnwidth}
        \includegraphics[width=\columnwidth]{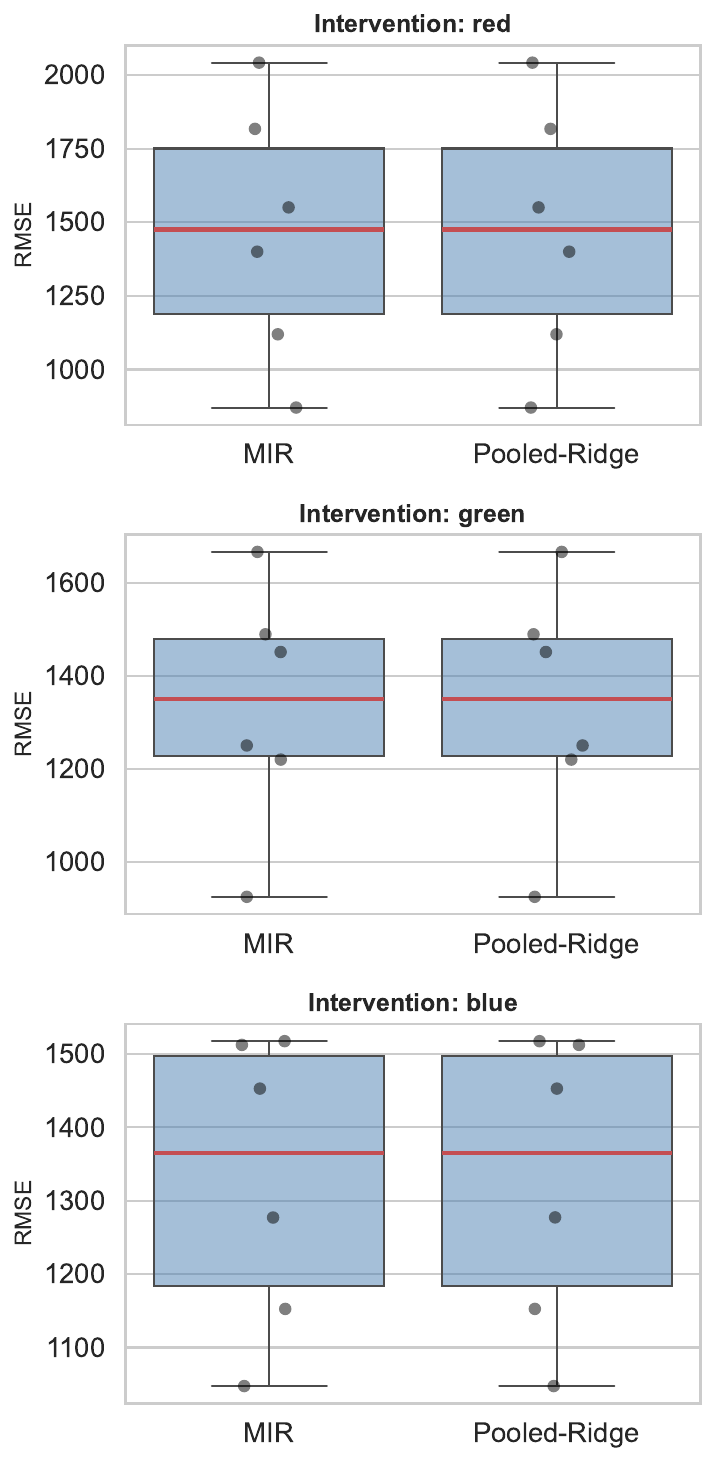}
        \subcaption{Reverse prediction direction on the Causal Chamber
        }
        \label{fig:chamber_reverse}
    \end{subfigure}
    \end{minipage}%
\caption{Stress-tests on violations of causal assumptions}
\label{fig:stress_test}
\end{figure}

\paragraph{Causal Chamber: reverse prediction direction.}
We use the same Light Tunnel setup as in Section~\ref{sec:causal_chamber} and Appendix~\ref{app:light_tunnel}: the same six environments per outcome-intervention pair, constructed by binning the intervention variable, and the same leave-one-environment-out evaluation protocol. The only change is the prediction direction: we now predict a light-intensity measurement (\texttt{ir\_1}) from the RGB brightness settings, rather than predicting a brightness setting from the sensor measurements. Since the brightness settings causally affect the sensor readings, this configuration violates the anti-causal assumption. For simplicity, we restrict to the fully-labeled regime $n_{\text{labeled}} = 5$, so all training environments have labels. We compare \texttt{MIR} against \texttt{Pooled-Ridge} using the same hyperparameter grid and selection procedure described in Appendix~\ref{app:light_tunnel}. Results are shown in Figure~\ref{fig:chamber_reverse}.

\paragraph{VitalDB without subject-centering.}
We use the same VitalDB setup as in Section~\ref{sec:vital_db} and Appendix~\ref{app:vitaldb}: the same 128 subjects, the same 140-dimensional APW embeddings as predictors, the same leave-one-subject-out evaluation protocol with $n_{\text{labeled}} \in \{20, 50, 80, 127\}$, and the same nMSE metric. The only change is that we omit the subject-centering preprocessing step, so $X$ and $Y$ retain their absolute values. In this setting, baseline physiological differences across subjects induce a direct effect of the environment $E$ on the outcome $Y$, violating Setting~\ref{setting:anti-causal}. To account for both mean and covariance shifts that may now be present, we use the combined MIR-VIR objective (see Appendix~\ref{app:MIR_VIR}), with separate regularization parameters $\gamma_1$ and $\gamma_2$ each searched over $\{10^2, 10^3, 10^4, 10^5\}$ and selected via leave-one-subject-out cross-validation on the labeled training subjects. Results are shown in Figure~\ref{fig:vital_noCT}.

\section{Additional experiments on VitalDB}\label{app:additional_experiments}

\begin{figure}[!t]
    \centering
    \begin{minipage}{.6\linewidth}
    \centering
    \begin{subfigure}[b]{1\columnwidth}
        \includegraphics[width=\columnwidth]{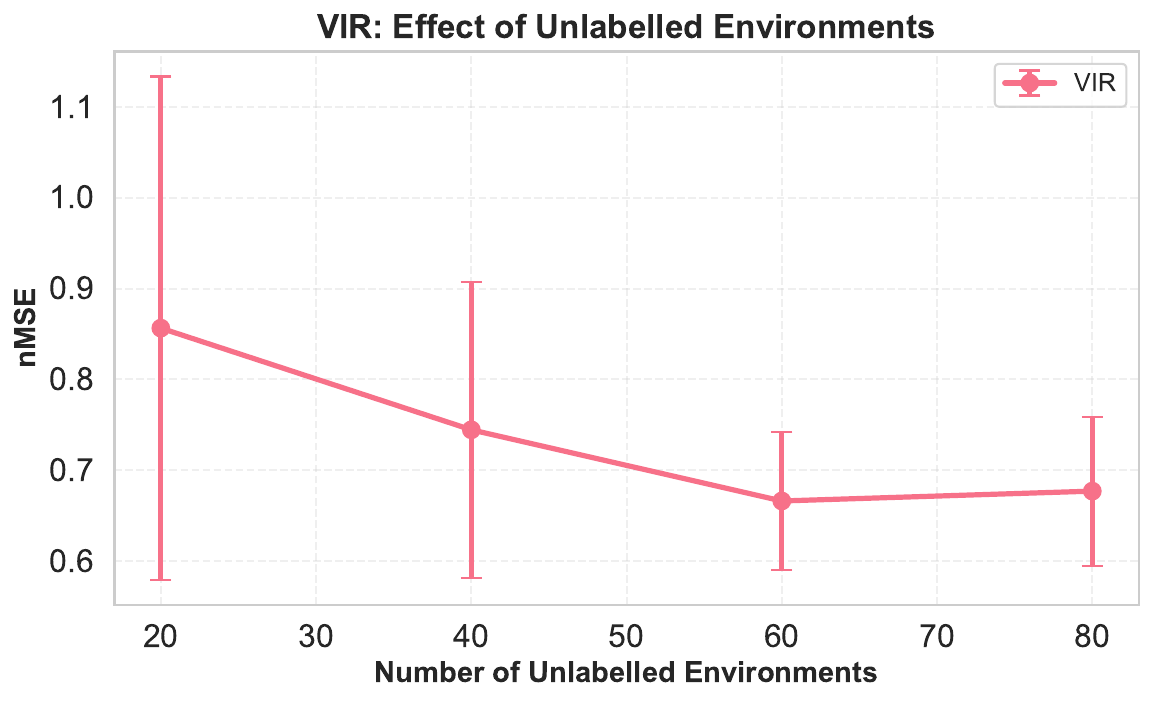}
    \end{subfigure}
    \end{minipage}%
\caption{Effect of unlabelled environments in VitalDB}
\label{fig:effect_unlabel}
\end{figure}

\paragraph{Effect of the number of unlabeled environments.} Fixing the number of labeled subjects at 20, we vary the number of additional unlabeled subjects from 0 to 80. As shown Figure~\ref{fig:effect_unlabel}, VIR's test nMSE decreases monotonically as more unlabeled subjects are added, providing direct evidence that unlabeled multi-environment data carries useful information for our regularization.

\section{Light Tunnel: Further experimental details}
\label{app:light_tunnel}

We use the Light Tunnel dataset \citep{gamella2025causal}, a controlled optical system designed for benchmarking machine learning methods. The system consists of RGB light sources and light-intensity sensors placed along a tunnel. The brightness settings of the light sources causally influence the sensor readings, providing a known causal ground truth.

\subsection{Data Description}

\paragraph{Predictors.} Each data point consists of 6 sensor readings:
\begin{itemize}
    \item Three infrared sensors: \texttt{ir\_1}, \texttt{ir\_2}, \texttt{ir\_3}
    \item Three visible light sensors: \texttt{vis\_1}, \texttt{vis\_2}, \texttt{vis\_3}
\end{itemize}
These sensor readings serve as predictors $X \in \mathbb{R}^6$.

\paragraph{Outcomes.} We consider three prediction tasks corresponding to the brightness settings of the three color channels: \texttt{red}, \texttt{green}, and \texttt{blue}. For each task, the outcome $Y \in \mathbb{R}$ is the brightness setting of the corresponding color channel.

\subsection{Experimental Design}

\paragraph{Outcome-Intervention Configurations.} For each outcome variable, we use data from the dataset's reference environment and intervention environments where the outcome is not directly intervened upon.

We test 6 configurations combining 3 outcome variables with 2 intervention variables each:
\begin{itemize}
    \item Outcome \textbf{red}: Use data from reference + blue/green intervention environments
    \item Outcome \textbf{blue}: Use data from reference + red/green intervention environments
    \item Outcome \textbf{green}: Use data from reference + red/blue intervention environments
\end{itemize}

This yields configurations: (red, blue), (red, green), (blue, red), (blue, green), (green, red), (green, blue), where the first element is the outcome and the second is the intervention variable used for environment discretization.

\paragraph{Environment Construction.} For each configuration, we construct 6 discrete environments as follows:
\begin{enumerate}
    \item Pool data from the reference environment and the corresponding intervention environments.
    \item Partition the pooled data into 6 bins based on the value of the intervention variable using uniform binning (equal-width intervals).
    \item Assign each data point to one of 6 environments $e_1, \ldots, e_6$ based on its bin.
\end{enumerate}
By construction, these environments differ in the mean of the intervention variable, which induces mean shifts in the predictor distribution.

\paragraph{Evaluation Protocol.} We use leave-one-environment-out evaluation:
\begin{enumerate}
    \item Hold out one environment as the test set.
    \item From the remaining 5 training environments, randomly select $n_{\text{labeled}} \in \{3, 4, 5\}$ environments to have labeled data (both $X$ and $Y$).
    \item The remaining training environments have only unlabeled data ($X$ only).
    \item Train models using labeled data from the selected environments; for \texttt{MIR}, additionally use unlabeled data from all training environments to estimate the regularization matrix.
    \item Evaluate on the held-out test environment using RMSE.
    \item Repeat steps 2--5 for 20 random selections of labeled environments.
\end{enumerate}
When $n_{\text{labeled}} = 5$, all training environments have labels, so no repeated trials are needed.

\subsection{Methods and Hyperparameters}

We compare the following methods:
\begin{itemize}
    \item \texttt{MIR}: Our Mean-based Invariant Regularization (Section~\ref{sec:MIR}).
    \item \texttt{MIR-Oracle}: \texttt{MIR} with the robustness parameter selected using the test environment (upper bound on \texttt{MIR} performance).
    \item \texttt{Pooled-Ridge}: Ridge regression pooling data across all labeled training environments.
    \item \texttt{AnchorReg}: Anchor regression \citep{rothenhausler2021anchor}.
    \item \texttt{GroupDRO}: Group distributionally robust optimization \citep{sagawa2020distributionally}.
\end{itemize}

\paragraph{Hyperparameter Grids.} We search over the following values:
\begin{itemize}
    \item \texttt{MIR} / \texttt{MIR-Oracle}: Robustness parameter $\gamma \in \{10^{-2}, 10^{-1}, 10^0, 10^1, 10^2, 10^3, 10^4, 10^5\}$.
    \item \texttt{Pooled-Ridge}: Ridge parameter $\alpha \in \{10^{-2}, 10^{-1}, 10^0, 10^1, 10^2, 10^3, 10^4, 10^5\}$.
    \item \texttt{AnchorReg}: Robustness parameter $\gamma \in \{10^{-2}, 10^{-1}, 10^0, 10^1, 10^2, 10^3, 10^4, 10^5\}$.
    \item \texttt{GroupDRO}: $\eta \in \{10^{-2}, 10^{-1}, 10^0, 10^1, 10^2, 10^3\}$.
\end{itemize}

\paragraph{Hyperparameter Selection.} For all methods except \texttt{MIR-Oracle}, we select hyperparameters using leave-one-environment-out cross-validation over the labeled training environments:
\begin{enumerate}
    \item For each hyperparameter value, perform leave-one-environment-out cross-validation over the $n_{\text{labeled}}$ labeled training environments.
    \item Compute the average MSE across folds.
    \item Select the hyperparameter with the lowest average MSE.
    \item Train the final model on all labeled training environments using the selected hyperparameter.
\end{enumerate}
For \texttt{MIR-Oracle}, the robustness parameter is selected based on performance on the held-out test environment.

\section{VitalDB: Further experimental details}
\label{app:vitaldb}

We use the VitalDB dataset \citep{lee2022vitaldb}, a collection of physiological time series recorded from surgical patients. The dataset contains vital sign measurements and hemodynamic parameters recorded during surgery.

\subsection{Data Description}

For each subject, we have a time series of arterial pressure waveform (APW) and stroke volume pairs. Each observation represents 8 seconds of APW signal and the corresponding average stroke volume over that interval. Different subjects have different total numbers of time steps, with an average of 362 time steps per subject.

\paragraph{Predictors.} For each subject, we extract embeddings from the APW signals using the pretrained CNN model described in Section~\ref{sec:vital_db}. Each 8-second APW segment is mapped to an embedding in $\mathbb{R}^{140}$.

\paragraph{Outcomes.} For each subject, the outcomes are stroke volume measurements, each corresponding to the average stroke volume over the corresponding 8-second interval.

\paragraph{Subject-centering.} We center both predictors and outcomes within each subject:
\begin{align*}
    X_i &\leftarrow X_i - \bar{X}_i, \\
    Y_i &\leftarrow Y_i - \bar{Y}_i,
\end{align*}
where $\bar{X}_i$ and $\bar{Y}_i$ denote the within-subject means computed over time steps. 

\subsection{Experimental Design}

\paragraph{Environment Definition.} Each subject defines one environment. With 128 subjects in total, we have 128 environments $e_1, \ldots, e_{128}$.

\paragraph{Evaluation Protocol.} We use leave-one-subject-out evaluation:
\begin{enumerate}
    \item Hold out one subject as the test environment.
    \item From the remaining 127 training subjects, randomly select $n_{\text{labeled}} \in \{20, 50, 80, 127\}$ subjects to have labeled data (both $X$ and $Y$).
    \item The remaining training subjects have only unlabeled data ($X$ only).
    \item Train models using labeled data from the selected subjects; for \texttt{VIR}, additionally use unlabeled data from all training subjects to estimate the regularization matrix.
    \item Evaluate on the held-out test subject.
    \item Repeat steps 2--5 for 20 random selections of labeled subjects.
\end{enumerate}
When $n_{\text{labeled}} = 127$, all training subjects have labels, so standard leave-one-subject-out cross-validation is performed without repeated trials.

\subsection{Methods and Hyperparameters}

We compare the following methods:
\begin{itemize}
    \item \texttt{VIR}: Our Variance-based Invariant Regularization (Section~\ref{sec:VIR}).
    \item \texttt{VIR-Oracle}: \texttt{VIR} with the robustness parameter selected using the test subject (upper bound on \texttt{VIR} performance).
    \item \texttt{Pooled-Ridge}: Ridge regression pooling data across all labeled training subjects.
    \item \texttt{GroupDRO}: Group distributionally robust optimization \citep{sagawa2020distributionally}.
    \item \texttt{DRIG}: Distributional Robustness via Invariant Gradients \citep{shen2025causality}.
\end{itemize}

\paragraph{Hyperparameter Grids.} We search over the following values:
\begin{itemize}
    \item \texttt{VIR} / \texttt{VIR-Oracle}: Robustness parameter $\gamma \in \{10^2, 10^3, 10^4, 10^5\}$.
    \item \texttt{Pooled-Ridge}: Ridge parameter $\alpha \in \{10^2, 10^3, \ldots, 10^{10}\}$.
    \item \texttt{DRIG}: Robustness parameter $\gamma \in \{10^{-2}, 10^{-1}, 10^0, 10^1, 10^2, 10^3, 10^4\}$.
    \item \texttt{GroupDRO}: $\eta \in \{10^{-2}, 10^{-1}, 10^0, 10^1, 10^2, 10^3\}$.
\end{itemize}

\paragraph{Hyperparameter Selection.} For all methods except \texttt{VIR-Oracle}, we select hyperparameters using leave-one-subject-out cross-validation over the labeled training subjects:
\begin{enumerate}
    \item For each hyperparameter configuration, perform leave-one-subject-out cross-validation over all $n_{\text{labeled}}$ labeled training subjects.
    \item Compute the average MSE across folds.
    \item Select the hyperparameter configuration with the lowest average MSE.
    \item Train the final model on all labeled training subjects using the selected hyperparameters.
\end{enumerate}
For \texttt{VIR-Oracle}, the robustness parameter is selected based on performance on the held-out test subject.

\paragraph{Prediction Post-Processing.} For all methods, we apply moving average smoothing with a window size of 4 minutes to both predictions and ground truth outcomes.

\end{document}